%% file: neurips_2025.tex
\documentclass{article}
\PassOptionsToPackage{dvipsnames}{xcolor}

\usepackage[preprint]{neurips_2025}

\usepackage[utf8]{inputenc} % allow utf-8 input
\usepackage[T1]{fontenc}    % use 8-bit T1 fonts
\usepackage{microtype}
\usepackage{setspace}
\usepackage[colorlinks,linkcolor=blue,anchorcolor=black,citecolor=blue]{hyperref}
\usepackage{url}            % simple URL typesetting
\usepackage{booktabs}       % professional-quality tables
\usepackage{amsfonts}       % blackboard math symbols
\usepackage{nicefrac}       % compact symbols for 1/2, etc.
\usepackage{algorithm}

\usepackage{pgfplots}
\pgfplotsset{compat=1.8}

\usepackage{multirow}
\usepackage{graphicx}
\usepackage{wrapfig}
\usepackage{epsfig}
\usepackage{adjustbox}
\usepackage{amsmath}
\usepackage{amssymb}
\usepackage{marvosym}
\usepackage{color}
\usepackage{threeparttable}
\usepackage{algorithm}
\usepackage{algpseudocode}
\usepackage{bbm}
\usepackage{caption}
\usepackage{subfig}
\usepackage{colortbl}
\usepackage{listings}
\usepackage{amsmath,amsthm}
\usepackage{pifont}

\usepackage{multirow} % 引入multirow包
\usepackage{colortbl}  %彩色表格需要加载的宏包
\usepackage{array}   %对表列和表格线的设置需要用到array宏包
\usepackage{bm}       % For bold math symbols

\usepackage{graphicx} % 用于 \scalebox
\usepackage{multirow} % 用于 \multirow
\usepackage{booktabs} % 用于 \toprule, \midrule, \bottomrule, \cmidrule
\newcommand{\ccol}{\hspace{0.1em}}

\usepackage{amsthm}
\usepackage{amsmath}

\newtheorem{theorem}{Theorem}
\newtheorem{lemma}{Lemma}
\newtheorem{proof}{Proof}[section]
\usepackage{soul} % 导入 soul 包

\definecolor{codegreen}{rgb}{0,0.6,0}
\definecolor{codegray}{rgb}{0.5,0.5,0.5}
\definecolor{codepurple}{rgb}{0.58,0,0.82}
\definecolor{backcolour}{rgb}{1.0,1.0,1.0}

\lstdefinestyle{mystyle}{
    backgroundcolor=\color{backcolour},   
    commentstyle=\color{codegreen},
    keywordstyle=\color{magenta},
    numberstyle=\tiny\color{codegray},
    stringstyle=\color{codepurple},
    basicstyle=\ttfamily\footnotesize,
    breakatwhitespace=false,         
    breaklines=true,                 
    captionpos=b,                    
    keepspaces=true,                 
    numbers=left,                    
    numbersep=5pt,                  
    showspaces=false,                
    showstringspaces=false,
    showtabs=false,                  
    tabsize=2,
    frame=single % 边框
}

\colorlet{my-red}{BrickRed!90!Sepia}
\colorlet{my-blue}{Aquamarine!30!Blue}
\hypersetup{
  colorlinks,
  linkcolor = my-red,
  citecolor = my-blue,
}

\title{Generalized Category Discovery via\\Token Manifold Capacity Learning}

\author{Luyao Tang$^{1}$, \quad Kunze Huang$^{1}$, \quad Chaoqi Chen$^{2}$, \quad Cheng Chen$^{3}$\\
{$^1$ Xiamen University}\quad
{$^2$ Shenzhen University}\quad
{$^3$ The University of Hong Kong}\\
{\tt\small  \{lytang, kzhuang\}@stu.xmu.edu.cn, cqchen1994@gmail.com, cchen@eee.hku.hk}
}

\begin{document}

\maketitle

\begin{abstract}
% Identifying previously unseen data is crucial for enhancing the robustness of deep learning models in the open world. Generalized category discovery (GCD) is a representative problem that requires clustering unlabeled data that includes known and novel categories. Current GCD methods mostly focus on minimizing intra-cluster variations, often at the cost of manifold capacity, thus limiting the richness of within-class representations. In this paper, we introduce a novel GCD approach that emphasizes maximizing the token manifold capacity (MTMC) within class tokens, thereby preserving the diversity and complexity of the data's intrinsic structure. Specifically, MTMC's efficacy is fundamentally rooted in its ability to leverage the nuclear norm of the singular values as a quantitative measure of the manifold capacity. MTMC enforces a richer and more informative representation within the manifolds of different patches constituting the same sample. MTMC ensures that, for each cluster, the representations of different patches of the same sample are compact and lie in a low-dimensional space, thereby enhancing discriminability. By doing so, the model could capture each class's nuanced semantic details and prevent the loss of critical information during the clustering process. MTMC promotes a comprehensive, non-collapsed representation that improves inter-class separability without adding excessive complexity.

Generalized category discovery (GCD) is essential for improving deep learning models’ robustness in open-world scenarios by clustering unlabeled data containing both known and novel categories. Traditional GCD methods focus on minimizing intra-cluster variations, often sacrificing manifold capacity, which limits the richness of intra-class representations. In this paper, we propose a novel approach, Maximum Token Manifold Capacity (MTMC), that prioritizes maximizing the manifold capacity of class tokens to preserve the diversity and complexity of data. MTMC leverages the nuclear norm of singular values as a measure of manifold capacity, ensuring that the representation of samples remains informative and well-structured. This method enhances the discriminability of clusters, allowing the model to capture detailed semantic features and avoid the loss of critical information during clustering. Through theoretical analysis and extensive experiments on coarse- and fine-grained datasets, we demonstrate that MTMC outperforms existing GCD methods, improving both clustering accuracy and the estimation of category numbers. The integration of MTMC leads to more complete representations, better inter-class separability, and a reduction in dimensional collapse, establishing MTMC as a vital component for robust open-world learning. \href{https://github.com/lytang63/MTMC}{Code is here}.

\end{abstract}

\input{secs/1_intro}
\input{secs/2_pre}

\input{secs/3_method}
\input{secs/4_exp}
\input{secs/6_conclusion}

\small
\bibliographystyle{plain}
\bibliography{main}

%%%%%%%%%%%%%%%%%%%%%%%%%%%%%%%%%%%%%%%%%%%%%%%%%%%%%%%%%%%%

\appendix

\input{secs/X_appendix}

\end{document}

%% file: secs/1_intro.tex
\section{Introduction}
\label{sec:intro}

% begin{wrapfigure}{行数}[位置][超出长度]{宽度}<图形>end{wrapfigure}
\begin{wrapfigure}{r}{7.5cm}
 \vspace{-0.4cm}
    \includegraphics[width=7.5cm, height=4cm]{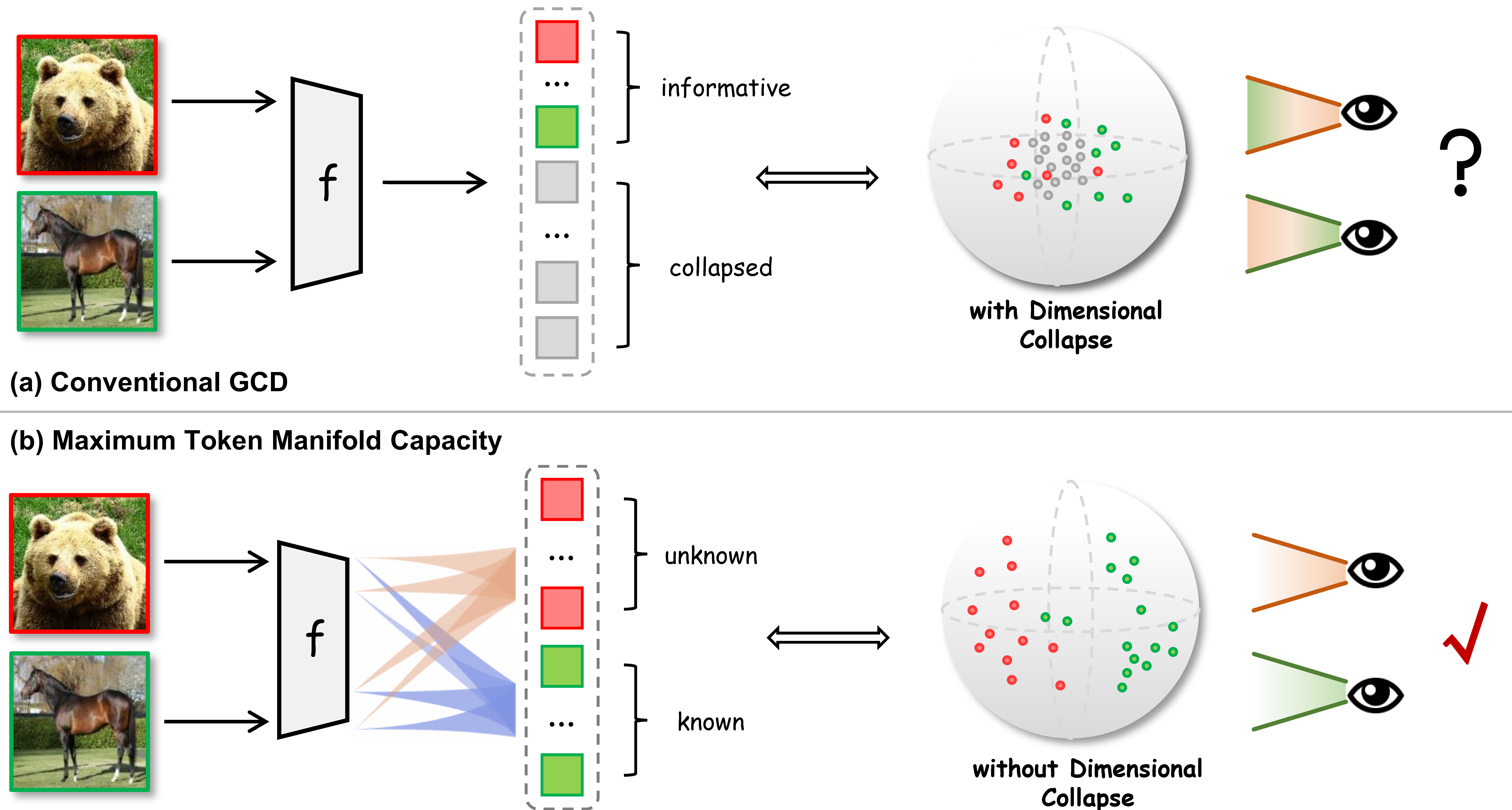}
    \caption{(a) GCD is constrained by dimensional collapse due to strong clustering, leading to mixed class features and limited representational capacity. (b) MTMC enhances the class token manifold capacity, improving representational completeness and unlocking the model's full potential in the open world.}
    \label{fig:clarify}
    \vspace{-0.4cm}
\end{wrapfigure}
Machine learning models encounter substantial challenges when deployed in real-world settings due to the intractability of objects in the open world~\cite{zhou2022domain, sarker2021machine, weiss2016survey}. The diversity of real-world objects exceeds the scope of data collected for training~\cite{wu2024towards}, and labeled data covers even fewer categories. Traditional deep learning models, trained on predefined categories, are ill-equipped to handle new category samples. To enhance the reliability of model deployment in real-world scenarios, open-world learning has emerged, aiming to identify and categorize unknown samples~\cite{han2019learning, geng2020recent, vaze2022generalized} in new environments.

% Machine learning models face a significant obstacle when moving towards deployment in the open world~\cite{zhou2022domain, sarker2021machine, weiss2016survey}, the complexity of the real world, specifically manifested in the uncertainty of the categories of objects. 
% One source of uncertainty comes from the fact that the categories of data collected and labeled for training cannot cover all objects in the real world~\cite{wu2024towards}. 
% The second uncertainty refers to the fact that humans typically predefine a set of categories for the model, with priors derived solely from human intelligence. 
% However, the attributes, scope, and inductive preferences of categories in the real world are unpredictable~\cite{han2019learning, geng2020recent, vaze2022generalized}. 
% Human biased definitions of the open world are disastrous for models that have only seen predefined samples. 
% Open-world learning has emerged to make models more reliably deployed in wild scenarios, intending to identify unknown category samples in new environments.

A plethora of approaches have been proposed to identify and categorize unknown samples, %have been proposed to separate known and unknown classes in the open world, 
such as open-set recognition (OSR)~\cite{geng2020recent} and novel class discovery (NCD)~\cite{han2019learning}. 
However, OSR treats all unknown samples as a single category. %NCD deals with partially labeled datasets, aiming to discover new classes in the unlabeled set. 
%Still, 
On the other hand, NCD relies on a strong assumption that all unlabeled samples encountered come from new classes. 
To relax this assumption, Generalized Category Discovery (GCD)~\cite{vaze2022generalized} permits the presence of known classes within unlabeled data. GCD relies on contrastive learning~\cite{choi2024contrastive} or prototype learning~\cite{wen2023parametric} to reduce the distance between semantically identical samples in the embedding space. 
However, current approaches face significant challenges: \textit{(i)} the compressed inter-class distribution may lead to the loss of useful information. This results in each cluster being unable to fully represent the semantic details within a class, leading to bias within the feature space, which is detrimental to category discovery. \textit{(ii)} Bias prevents the inter-class decision boundaries from aligning with the boundaries between real-world categories, making it impossible for the model to accurately separate clusters during the discovery of categories (Figure~\ref{fig:eigenvalue_acc} demonstrates that incomplete intra-class representations result in low clustering accuracy).

To this premise, we challenge the status quo by raising an open question: \textit{Can deep models accurately separate new semantics during the category discovery by enhancing the \textbf{completeness of intra-class representations}?} The GCD aims to partition data points into distinct clusters, which are distributed on low-dimensional manifolds~\cite{souvenir2005manifold, wah2011caltech} within high-dimensional spaces. Recently, Maximum Manifold Capacity Representations~\cite{yerxa2023learning, schaeffer2024towards, isik2023information} have sought to learn representations by examining the separability of manifolds. In this context, manifolds containing views of the same scene are both compact and low-dimensional, while manifolds corresponding to different scenes are separated.

Building on manifold capacity concept, we introduce Maximum Class Token Manifold Capacity (\textbf{MTMC}). Specifically, (1) we associate low intra-class representation completeness with low manifold capacity. Our research narrows the focus from the entire feature space to the intra-class feature space, examining manifold capacity at a more \textbf{granular token level}. (2) We consider the representation of a sample as its manifold, with the sample representation in GCD derived from the class token provided by Vision Transformer (ViT)~\cite{dosovitskiy2020image}. 
Under the attention mechanism, the class token refines the patch features, thus serving as a proxy for the \textbf{sample manifold}. (3) Given that a comprehensive and information-rich class token manifold necessitates a large capacity, we measure manifold capacity using the \textbf{nuclear norm of class token} and aim to maximize this norm. (4) MTMC enhances the completeness of sample representation, enabling clusters to capture more \textbf{intra-class semantic details} while preventing dimensionality collapse, thus improving inter-class separability accuracy. Our contributions are summarized as follows:

\begin{itemize}
% \item We propose a method called MTMC to enhance representation completeness, thereby empowering the model for generalized category discovery. We theoretically analyze the effectiveness of MTMC as a means to address dimensionality collapse and enhance representation quality.
% \item We increase the capacity of the class token manifold by maximizing the nuclear norm of the singular value kernel of the class token, allowing clusters to represent more intra-class semantic details.
% \item MTMC is simple to implement. Experiments on coarse-grained and fine-grained datasets prove the effectiveness of precision in category discovery and accuracy in estimating the number of categories.
\item We introduce MTMC to enhance representation completeness and analyze its effectiveness in addressing dimensional collapse and improving von Neumann entropy.
\item We maximize the nuclear norm of the class token's singular value kernel to increase its manifold capacity, enabling clusters to capture more intra-class semantic details.
\item MTMC is easy to implement. Experiments on coarse- and fine-grained datasets demonstrate its effectiveness in improving precision and the accuracy of category number estimation.

\end{itemize}

%% file: secs/2_pre.tex
\section{Preliminary and Motivation}
\label{sec:pre}

\subsection{Notation and Optimization of GCD}

% Given a labeled dataset $ \mathcal{D}_l = \{ (\mathbf{x}_i^l, y_i^l) \} \subset \mathcal{X} \times \mathcal{Y}_l $ and an unlabeled dataset $ \mathcal{D}_u = \{ (\mathbf{x}_i^u, y_i^u) \} \subset \mathcal{X} \times \mathcal{Y}_u $, the dataset $ \mathcal{D}_l $ includes only old classes, whereas $ \mathcal{D}_u $ comprises both old and new classes, i.e., $ \mathcal{Y}_l = \mathcal{C}_{old} $ and $ \mathcal{Y}_u = \mathcal{C}_{old} \cup \mathcal{C}_{new} $. Models are tasked with clustering both old and new classes in $ \mathcal{D}_u $. The number of novel classes $ K_{new} $ is either known a priori or estimated \cite{vaze2022generalized, pu2023dynamic, Zhao_2023_ICCV}. The functions $ f(\cdot) $ and $ g(\cdot) $ serve as the feature extractor and projection head for contrastive learning, respectively. The characteristic $ \mathbf{h}_i = f(\mathbf{x}_i) $ and the projected embeddings $ \mathbf{z}_i = g(\mathbf{h}_i) $ are normalized l-2.

For each dataset, consider a labeled subset $ \mathcal{D}_l = \{ (\mathbf{x}_i^l, y_i^l) \} \subset \mathcal{X} \times \mathcal{Y}_l $ and an unlabeled subset $ \mathcal{D}_u = \{ (\mathbf{x}_i^u, y_i^u) \} \subset \mathcal{X} \times \mathcal{Y}_u $. Only known classes can be found in $ \mathcal{D}_l $, while $ \mathcal{D}_u $ encompasses known and novel classes, translating to $ \mathcal{Y}_l = \mathcal{C}_{known} $ and $ \mathcal{Y}_u = \mathcal{C}_{known} \cup \mathcal{C}_{novel} $. The task of models involves clustering on both the known and novel classes in $ \mathcal{D}_u $. The number of novel classes represented as $ K_{novel} $ can be determined beforehand~\cite{vaze2022generalized, pu2023dynamic, zhao2023learning}. The functions $ f(\cdot) $ and $ g(\cdot) $ perform as the feature extractor and projection head, respectively. Both the feature $\mathbf{h}_i = f(\mathbf{x}_i)$ and the projected embedding $ \mathbf{z}_i = g(\mathbf{h}_i) $ are under L-2 normalization.

For compact clustering, GCD consists of supervised and unsupervised contrastive learning or prototype learning (discussed in Appendix \ref{subsec:rela_gcd}). While these methods achieve clustering, they overly prioritize compact clusters, overlooking incomplete intra-class representations that fail to capture the full distribution, resulting in low manifold capacity. More details are provided in Appendix \ref{sec:gcd}.

\subsection{Manifold Capacity Theory}
Manifold capacity theory~\cite{yerxa2023learning,  isik2023information} evaluates the efficiency of neural representation coding by mapping high-dimensional data to low-dimensional manifolds representing different objects or categories. Key concepts include: \textit{(1) Manifold Radius}: $R\_M = \sqrt{\frac{1}{P} \sum\_{i=1}^P \lambda\_i^2}$, where $\lambda\_i$ are the eigenvalues of the covariance matrix of points on the manifold, and $P$ is the number of points. It measures the manifold's size relative to its centroid.
\textit{(2) Manifold Dimensionality}: $D\_M = \frac{\left(\sum\_{i=1}^P \lambda\_i\right)^2}{\sum\_{i=1}^P \lambda\_i^2}$, quantifying the manifold's expansion along its major directions.
\textit{(3) Manifold Capacity}: $\alpha\_C = \phi\left(R\_M \sqrt{D\_M}\right)$, where $\phi(\cdot)$ is a monotonically decreasing function. This represents the maximum number of linearly separable manifolds in a feature space. Manifold capacity, derived from radius and dimensionality, determines the number of distinguishable categories in high-dimensional space. Optimizing it enhances coding efficiency by refining the manifold's geometric properties.

% Formally, it can be represented as $\mathbf{Z}^*=\arg \min _\mathbf{Z}\|\mathbf{G} \mathbf{Z}\|_*$, where $\mathbf{G}$ is a symmetric matrix that encodes the semantic relationship between different augmented views, and $\mathbf{Z}$ is an embedding matrix that contains the representation vectors of all augmented views. The optimal embedding $\mathbf{Z}^*$ is the one that minimizes the nuclear norm of $\mathbf{G} \mathbf{Z}$.

\subsection{Our Thoughts on Why GCD Needs a Manifold Capacity Quest}

Based on above, GCD has made notable progress in clustering both known and unknown categories. However, current methods often fall short due to their focus on compact clustering, which compromises the richness of intra-class representations. We identify two key issues that motivate the need for manifold capacity quest in GCD:

\textbf{Incomplete Intra-class Representations}: Existing methods overlook the need for a comprehensive representation within each class, resulting in poor feature embeddings that fail to capture the full diversity of category-specific structures.

\textbf{Dimensional Collapse}: Compact clustering methods lead to dimensional collapse, where embeddings become overly simplified and fail to preserve the data's intrinsic complexity. This limits the ability of GCD models to accurately separate categories.

Through analysis (detailed proofs in Appendix \ref{sec:theoremMTMC}), we demonstrate that MTMC provides a theoretical framework that can substantially improve the accuracy and robustness of GCD, making it an essential addition for real-world category discovery.

%% file: secs/3_method.tex
\section{Methodology}
\label{sec:method}

% Maximum manifold capacity focuses on the manifold on a per-class basis. Thus, the target is the samples and their augmented views, aiming to reduce the distances between different views. This relies on a large amount of data augmentation for the same sample. GCD has already brought the data points of the same category closer in the feature space through strong clustering methods. However, the completeness of the representations within the samples forming a cluster is directly overlooked. Such collapsed representations are of low manifold capacity and low information content, significantly limiting the upper bound of the model.

% Based on the limitations above, our research focuses on the manifold on a per-sample basis, enhancing the sample-level manifold capacity and unleashing the potential of the model in representation abilities. Therefore, we turn our attention to the class token of the most representative samples, which carry the most informative content. 

\begin{figure}[t]
\begin{center}
  \includegraphics[width=0.85\linewidth]{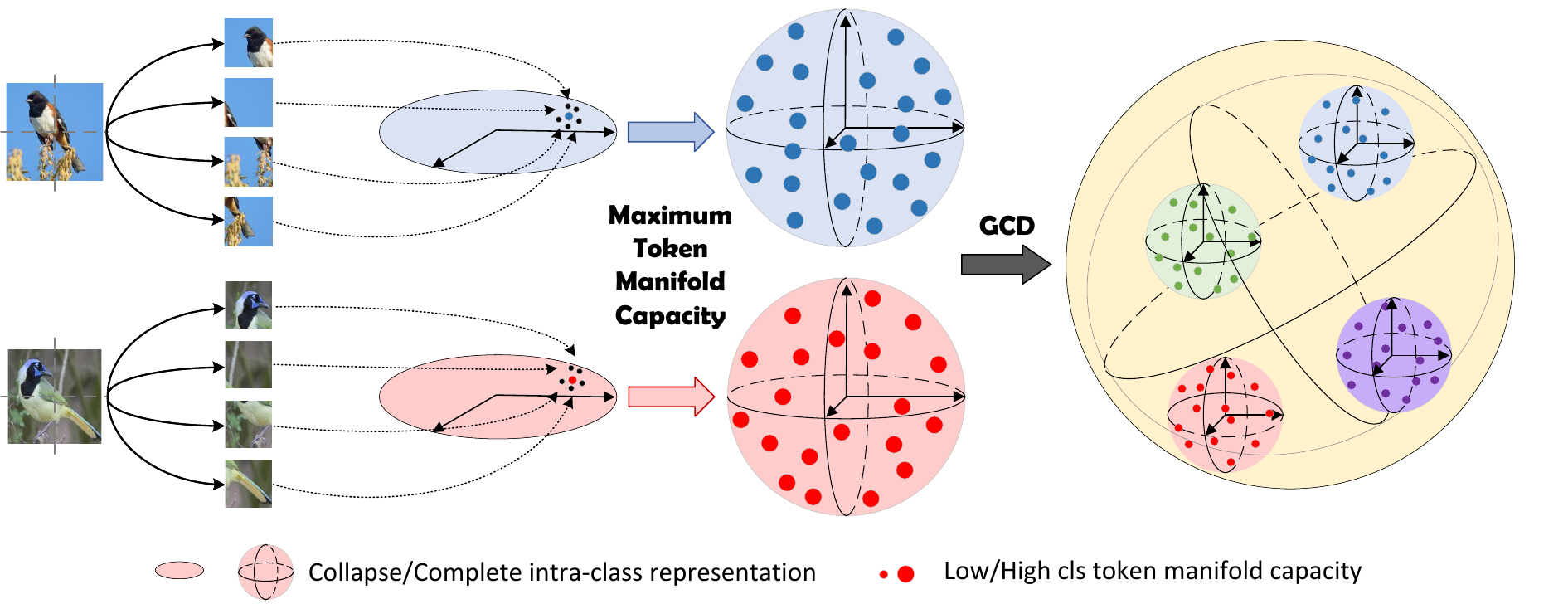}
\end{center}
\caption{Overview of Maximum Token Manifold Capacity.}
\label{fig:overview}
\end{figure}

As shown in Figure~\ref{fig:overview}, Maximum Token Manifold Capacity is pithy. For simplicity, we use \texttt{[cls]} to represent the class token and \texttt{[vis]} to represent visual/patch tokens. In Section~\ref{subsec:extent}, we trace the formation process of \texttt{[cls]} and \texttt{[vis]}, and identify \texttt{[cls]} as the sample centroid, also providing the definition of class token manifold extent, which is strongly correlated with capacity. In Section~\ref{subsec:mtmc}, we introduce the optimization objective of maximum class token manifold capacity and offer a concise code illustration.

\subsection{Extent of Class Token Manifold}
\label{subsec:extent}

We introduce the concept of ``sample centroid" without imposing restrictions on backbone, whether they are CNNs or Transformers. In the GCD task, the backbone network is ViT, and the \texttt{[cls]} is treated as the "sample centroid" refined from \texttt{[vis]}. Mathematically, the refined sample centroid can be described as the weighted average of all visual tokens using a self-attention mechanism. Here, the sample centroid refers to the weighted aggregation of features from all visual tokens by the class token through a self-attention mechanism to form the global representation of the image. The concepts of \textbf{sample centroid manifold} and \textbf{class token manifold} are equivalent in nature.

Specifically, in the self-attention layer of the Transformer, each token (including \texttt{[cls]} and \texttt{[vis]}) calculates attention scores with respect to all other tokens. These attention scores are used to weight the features of each visual token for updating the class token. The self-attention mechanism can be represented as $\operatorname{Attention}(\mathbf{q}, \mathbf{k}, \mathbf{v})=\operatorname{softmax}\left(\frac{\mathbf{q k}^\top}{\sqrt{d_k}}\right) \mathbf{v}$. The $\mathbf{q}, \mathbf{k}, \mathbf{v}$ represent the query, key, and value matrices, respectively. These matrices are generated from the embedding vectors of tokens through linear layers. $d_k$ is the square root of the dimension of the key vectors. It is used to scale the dot products to prevent gradient vanishing or exploding.

For the class token, its update can be represented as:

\begin{equation}
\mathbf{\texttt{[cls]}}^{\prime}=\operatorname{Attention}(\mathbf{\texttt{[cls]}}, \mathbf{k}, \mathbf{v})+\mathbf{\texttt{[cls]}},
\end{equation}

where $\mathbf{\texttt{[cls]}}^{\prime}$ represents the updated class token embedding, and $+$ denotes the residual connection. In the self-attention mechanism, the update of the class token can be seen as the weighted average of the features of all patch tokens, where the attention scores determine the weights:

\begin{equation}
\mathbf{\texttt{[cls]}}^{\prime}=\sum_{i=1}^{H \times W} \alpha_i \mathbf{\texttt{[vis]}}_i+\mathbf{\texttt{[cls]}}.
\label{eq:cls}
\end{equation}

The $\alpha_i$ represents the attention score of the class token to the $i$-th patch token and $\mathbf{\texttt{[vis]}}_i$ denotes the embedding vector of the $i$-th patch token. The class token can be regarded as the weighted average of the features of all patch tokens, known as the "sample centroid," where the self-attention mechanism dynamically computes the weights. This weighted average allows the class token to capture the global features of the image, rather than just a simple arithmetic mean.

Given \texttt{[vis]} and \texttt{[cls]}, the extent of the sample centroid manifold, also known as the class token manifold extent (CTME), can be represented as: 

\begin{equation}
CTME = \|\texttt{[cls]}\|_*,
\label{eq:ctme}
\end{equation}

where $\|\cdot\|_*$ represents the nuclear norm. The sample centroid manifold contains the magnitudes of each individual visual/patch token manifold. If Equation~\ref{eq:ctme} is considered as the optimization objective, that is, when the sample centroid manifold is maximized, it implicitly minimizes each \texttt{[vis]} manifold, thereby enhancing the intra-manifold similarity. Further understanding is provided in Section~\ref{subsec:mtmc}.

\subsection{Maximum Class Token Manifold Capacity}
\label{subsec:mtmc}

This subsection provides a detailed description of Maximum Class Token Manifold Capacity. Specifically, for the labeled samples provided in the GCD task, we assume that the annotations provided by human annotators are sufficiently accurate and unbiased. Therefore, supervised methods can effectively shape the manifold of these samples. As a result, we focus on enhancing the manifold capacity of the unlabeled samples.

The functions $ f(\cdot) $ and $ g(\cdot) $ perform as the feature extractor and projection head, respectively. Both the feature $\mathbf{h}_i = f(\mathbf{x}_i)$ and the projected embedding $ \mathbf{z}_i = g(\mathbf{h}_i) $ are under L-2 normalization.

For the unlabeled samples in the mini-batch $\mathcal{B}^{u}$, after the feature extractor cuts them into $H \times W$ patches, the features are sent to the projection layer to obtain embeddings, which are the visual tokens of unlabeled samples:

\begin{equation}
\texttt{[vis]}^u = \mathbf{z}_i^u \stackrel{\text { def }}{=} g(f(\mathbf{x}_i^u)) \in \mathcal{Z} ,
\label{eq:visu}
\end{equation}

where, $\mathcal{Z}$ is commonly the $D$-dimensional hypersphere $\mathbb{S}^{D-1} \stackrel{\text { def }}{=}\left\{\mathbf{z} \in \mathbb{R}^D: \mathbf{z}^T \mathbf{z}=1\right\}$ or $\mathbb{R}^D$.

Furthermore, from Equation~\ref{eq:cls}, we can obtain the refined sample centroid that represents \texttt{[vis]}$^u$, which is denoted as \texttt{[cls]}$^u$, and define the loss function for maximum class token manifold capacity:

\begin{equation}
\mathcal{L}_\text{{MTMC}} \stackrel{\text { def }}{=}-\|\texttt{[cls]}^u\|_* \stackrel{\text { def }}{=}-\sum_{r=1}^{\operatorname{rank}(\texttt{[cls]}^u)} \sigma_r(\texttt{[cls]}^u),
\label{eq:mtmc}
\end{equation}

where $\sigma_r(\texttt{[cls]}^u)$ is the $r$-th singular value of \texttt{[cls]}$^u$.

Minimizing the MTMC loss maximizes the nuclear norm of the class token. Without MTMC, the manifold within clusters has a larger range, resulting in a lower nuclear norm of the centroid matrix. Geometrically, $\texttt{[vis]}$ manifolds represent subspaces in high-dimensional space, with each corresponding to the value range of a slice feature. Maximizing CTME ensures that the class token $\texttt{[cls]}$ finds the most representative "center" in the space of $\texttt{[vis]}$ manifolds, minimizing the distance (reflected in the nuclear norm) from all $\texttt{[vis]}$ to this "center." This increases the nuclear norm of the centroid matrix and enhances the representation by unraveling collapsed representations.

The MTMC implementation is concise, comprising only three lines. After calculating the GCD loss $\mathcal{L}_{\text{GCD}}$, the class token is obtained, singular value decomposition is performed, and the sum of singular values is added to the loss, resulting in $\mathcal{L}_{\text{GCD}} + \lambda \mathcal{L}_{\text{MTMC}}$.

\begin{lstlisting}[language=Python]
def forward(self, x_unlabel, loss):
    f_unlabel = self.featurizer(x_unlabel) # cls and vis tokens
    f_cls_unlabel = f_unlabel[:,0] # get cls token
    z_cls_unlabel = self.projector(f_cls_unlabel) #embedding
    _,s,_ = torch.svd(z_cls_unlabel) # singular value decomposition
    loss += self.lambda * torch.sum(s) # MTMC
    return loss
\end{lstlisting}

\subsection{Maximum Class Token Manifold Capacity Increases Von Neumann Entropy} % figure 6

\begin{wrapfigure}{r}{6cm}
 \vspace{-10pt}
    \includegraphics[width=1.0\linewidth]{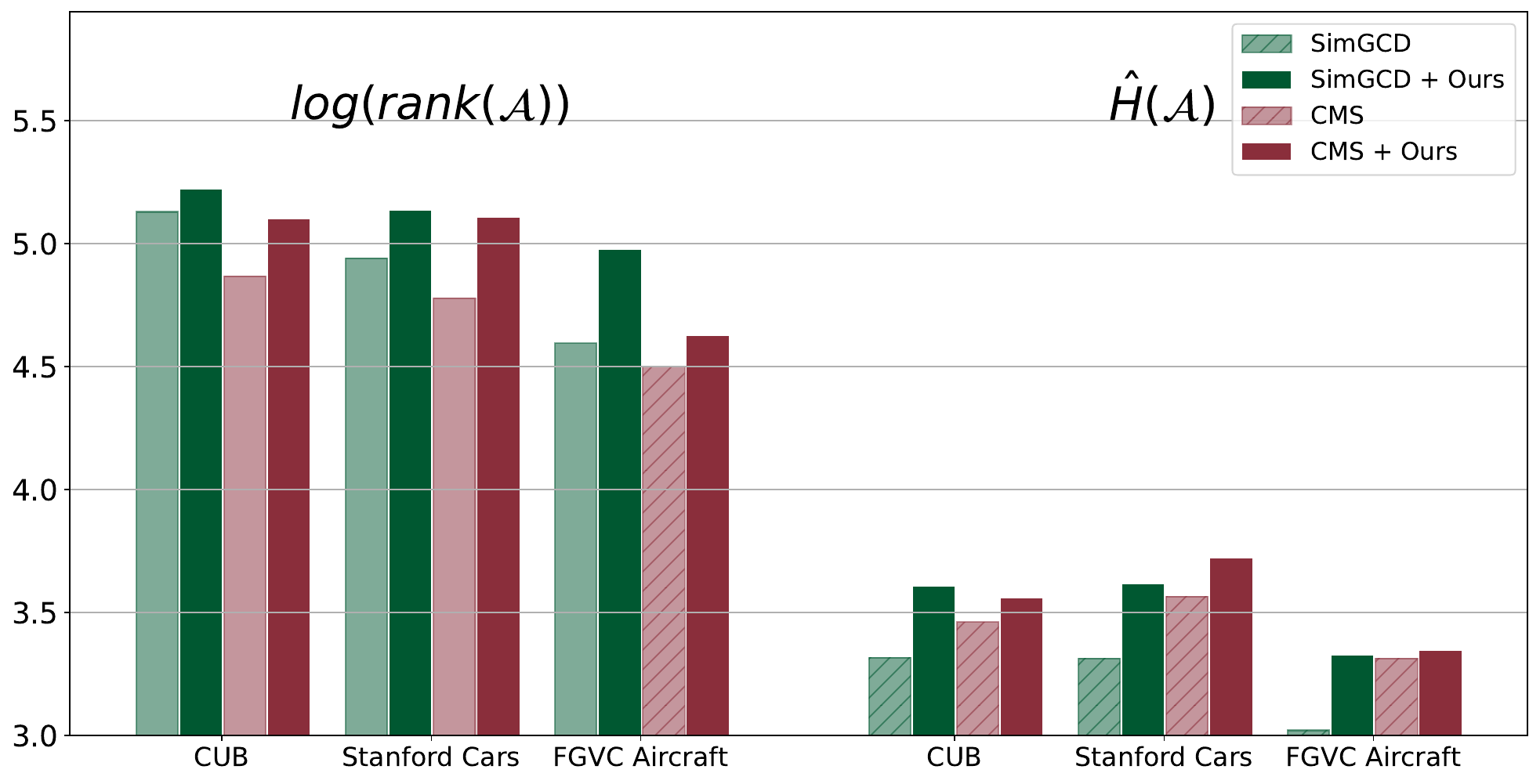}
    \caption{Comparison between $log(\operatorname{rank}(\mathcal{A}))$ and $\hat{H}(\mathcal{A})$. The count of the largest eigenvalues necessary to account for 99\% of the total eigenvalue energy serves as a surrogate for the rank.}
    \label{fig:entropy_logrank}
    \vspace{-0.7cm}
\end{wrapfigure}
The autocorrelation matrix of the sample's class token manifold is $\mathcal{A} \triangleq \sum_{i=1}^N \frac{1}{N} \mathbf{\texttt{[cls]}}_i \mathbf{\texttt{[cls]}}_i^{\top}=\mathbf{CLS}^{\top} \mathbf{CLS} / N$. We employ von Neumann entropy~\cite{petz2001entropy,boes2019neumann} to measure manifold capacity. This gives the advantage of focusing exclusively on the eigenvalues obtained after decomposition, allowing for graceful handling of eigenvalues that are extremely close to $0$. The von Neumann entropy can be expressed as $\hat{H}\left(\mathcal{A}\right) \triangleq-\sum_j \lambda_j \log \lambda_j$, representing the Shannon entropy of the eigenvalues of $\mathcal{A}$, with values ranging between $0$ and $\log d$. A larger $\hat{H}(\mathcal{A})$ indicates a greater manifold capacity of the features. 

% \begin{figure}[h]
% \centering
%   \includegraphics[width=0.7\linewidth]{figs/entropy_logrank.pdf}
% \caption{Comparison between $log(\operatorname{rank}(\mathcal{A}))$ and $\hat{H}(\mathcal{A})$. The count of the largest eigenvalues necessary to account for 99\% of the total eigenvalue energy serves as a surrogate for the rank.}
% \label{fig:entropy_logrank}
% \end{figure}

Von Neumann entropy is an effective measure for assessing the uniformity of distributions and managing extreme values. As illustrated in Figure~\ref{fig:entropy_logrank}, the incorporation of MTMC results in a von Neumann entropy for the feature embeddings that is significantly higher than that of the original scheme. Furthermore, it is possible to relate von Neumann entropy to the rank of the \texttt{[cls]}. When $\mathcal{A}$ possesses uniformly distributed eigenvalues with full rank, the entropy is maximized, which can be explicitly expressed as below.

\begin{theorem}
For a given \texttt{[cls]} autocorrelation $\mathcal{A} =\mathbf{CLS}^{\top} \mathbf{CLS} / N  \in \mathbb{R}^{d \times d}$ of rank $k$ ($\leq d$),
\begin{equation}
    \log \left(\operatorname{rank}\left(\mathcal{A}\right)\right) \geq \hat{H}\left(\mathcal{A}\right)
\end{equation}
where equality holds if the eigenvalues of $\mathcal{A}$ are uniform  with $\forall_{j=1}^k \lambda_j=1 / k$ and $\forall_{j=k+1}^d \lambda_j=0$ .
\end{theorem}

The details of the proof process are in Appendix \ref{sec:theorem}. A higher von Neumann entropy generally implies a larger manifold capacity. We provide a comparison of the $log(\operatorname{rank}(\mathcal{A}))$ and $\hat{H}(\mathcal{A})$ for different schemes in Figure~\ref{fig:entropy_logrank}, and it can be observed that MTMC has a higher value, indicating the high-rank nature of the features and the uniformity of neuron activation in each dimension of representation.

%% file: secs/4_exp.tex
\section{Experiments}
\label{sec:mainexp}

\input{tabs/gcd}

\subsection{Setup}

\textbf{Benchmarks}. MTMC is evaluated on coarse- and fine-grained benchmarks. These include two conventional datasets, CIFAR100~\cite{krizhevsky2009learning} and ImageNet100~\cite{geirhos2018imagenettrained}, and four fine-grained datasets, CUB-200-2011~\cite{wah2011caltech}, Stanford Cars~\cite{krause20133d}, FGVC Aircraft~\cite{maji2013fine}, and Herbarium19~\cite{tan2019herbarium}. To segregate target classes into sets of known and unknown, we adhere to the splits defined by the Semantic Shift Benchmark~\cite{vaze2021open} when working with CUB, Stanford Cars, and FGVC Aircraft. The splits from the previous study~\cite{vaze2022generalized} is employed for the remaining datasets, we designate 80\% of the classes as known under the CIFAR100 benchmark. For the rest of the benchmarks, the proportion of known classes stands at 50\%. Our labeled set $\mathcal{D}_{l}$, comprises 50\% images from the known classes for all benchmarks.

% Comprehensive details about data splits can be referred to in Table~\ref{table:split} in Appendix Sec.~\ref{sec:detail_dataset}.

\textbf{Evaluation Protocols}. 
We assess MTMC's effectiveness via a two-step process. First, we cluster the complete collection of images defined as $\mathcal{D}$. Then, we measure the accuracy on the set $\mathcal{D}_{u}$. In line with previous research~\cite{vaze2022generalized}, accuracy is determined by comparing the assignments to the actual labels using the Hungarian optimal matching~\cite{kuhn1955hungarian}. This method bases the match on the number of instances that intersect between each pair of classes. Instances that do not belong to any pair, \textit{i.e.}, unpaired classes, are viewed as incorrect predictions. On the other hand, instances belonging to the most abundant class within each ground-truth cluster are taken as correct for accuracy calculations. We present the accuracy for all unlabeled data, and the accuracy is classified as old/known and new/novel, respectively. The accuracy using the estimated number of classes and the ground-truth $K$ are reported. This allows us to compare MTMC with previous studies that have assumed the availability of the $K$ during the evaluation phase.

\textbf{Implementation Details}. 
The purpose of MTMC is to empower existing GCD schemes to improve the completeness of representation. We closely adhere to their initial implementation details for an effective comparison. We use a pre-trained DINO ViT-B/16~\cite{caron2021emerging, dosovitskiy2020image}, utilizing it as our image encoder along with a projection head, an approach consistent with existing methods~\cite{vaze2022generalized, zhang2023promptcal, pu2023dynamic}. All of our experiments are performed with a single NVIDIA RTX4090. \textbf{We follow the original training parameter details of each scheme to illustrate the generality and applicability of MTMC.} The count of the largest eigenvalues
necessary to account for 99\% of the total eigenvalue energy serves as a surrogate for the rank in Equation~\ref{eq:mtmc}.

% \subsection{Baselines}

\subsection{Main Results}

\input{tabs/gcd2}

\textbf{Evaluation on GCD.} As shown in Tables~\ref{tab:gcd} and~\ref{tab:gcd2}, MTMC brings consistent and notable gains across all evaluated GCD methods and datasets, under both known and unknown class number settings. Key findings are as follows: 
\textit{\textbf{\ding{182} Compatibility.}} MTMC improves all baselines including SimGCD, CMS, SPTNet, and SelEx without any architectural changes or tuning. For example, on CUB with known class number, MTMC enhances SimGCD by 2.6\% and SPTNet by 1.9\%. On ImageNet100, it improves CMS by 3.2\% in the All setting and boosts SelEx by 2.8\% on novel classes. These results highlight MTMC’s strong generalization across frameworks and confirm its plug-and-play compatibility.
\textit{\textbf{\ding{183} Generality.}} MTMC yields stable gains on both coarse-grained datasets like CIFAR100 and ImageNet100 and fine-grained ones like CUB and Cars. Notably, on Stanford Cars, MTMC improves CMS by 3.0\% and SelEx by 2.6\% under unknown class number settings. Average improvements on novel classes range from 1.4\% to 2.2\% across datasets, demonstrating robustness to domain complexity and label granularity.
\textit{\textbf{\ding{184} Correctness.}} By maximizing manifold capacity, MTMC enhances intra-class representation completeness and inter-class separation, leading to improved clustering under various scenarios. The consistent gains across all baselines and benchmarks validate our theoretical view that capacity-aware representation learning is a principled and effective direction for solving GCD.
In sum, MTMC’s universal improvements across models and datasets affirm the correctness of our capacity-based view and establish it as a general and practical solution for GCD.

\begin{figure}[tbp]
  \centering
  \vspace{-0.4cm}
  \subfloat[CUB]
  {\includegraphics[width=0.30\textwidth]{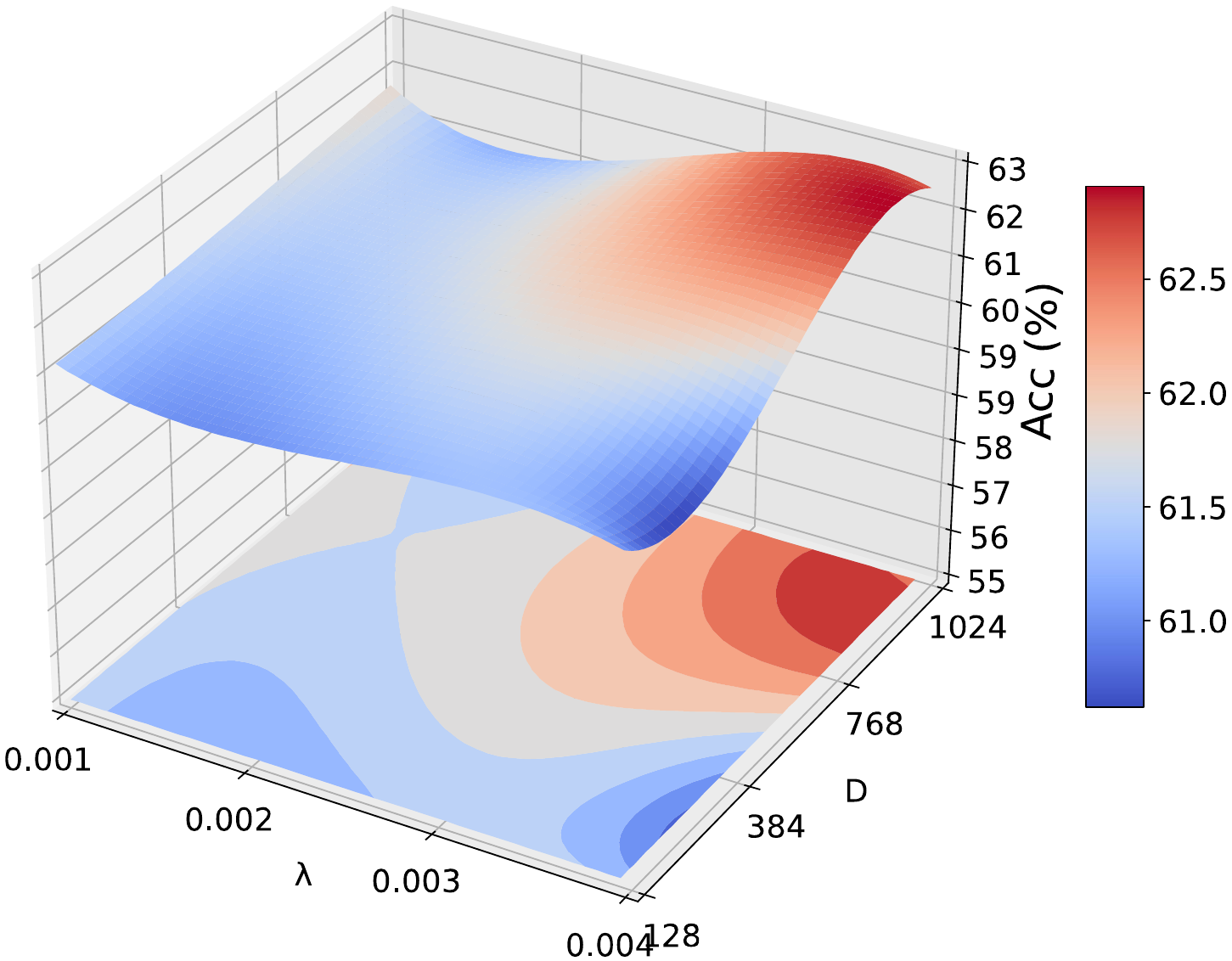}}
  \quad     % 重点就在这，优先横向排列，自动换行
  \subfloat[Stanford Cars]
  {\includegraphics[width=0.30\textwidth]{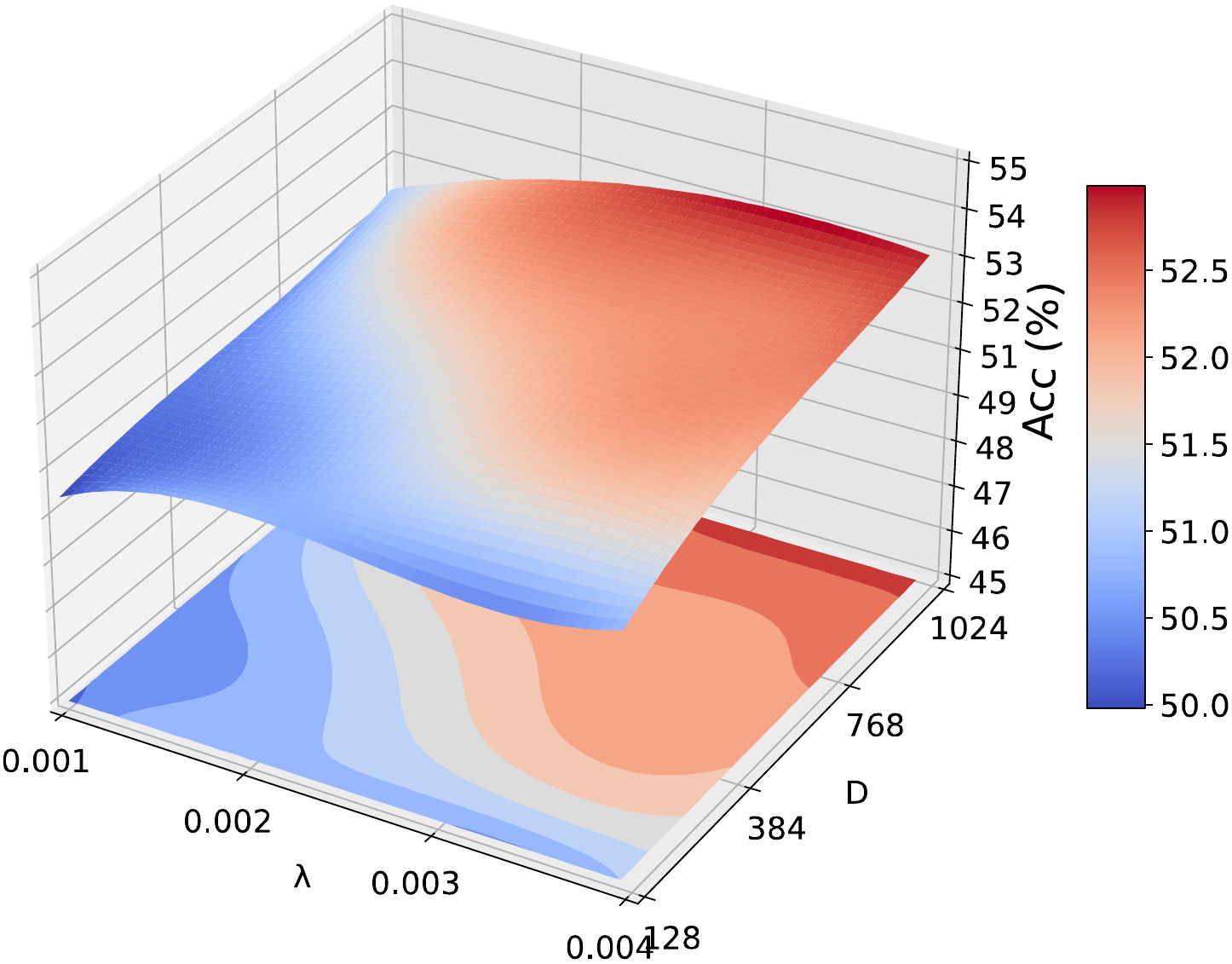}}
  \quad
  \subfloat[FGVC Aircraft]
  {\includegraphics[width=0.30\textwidth]{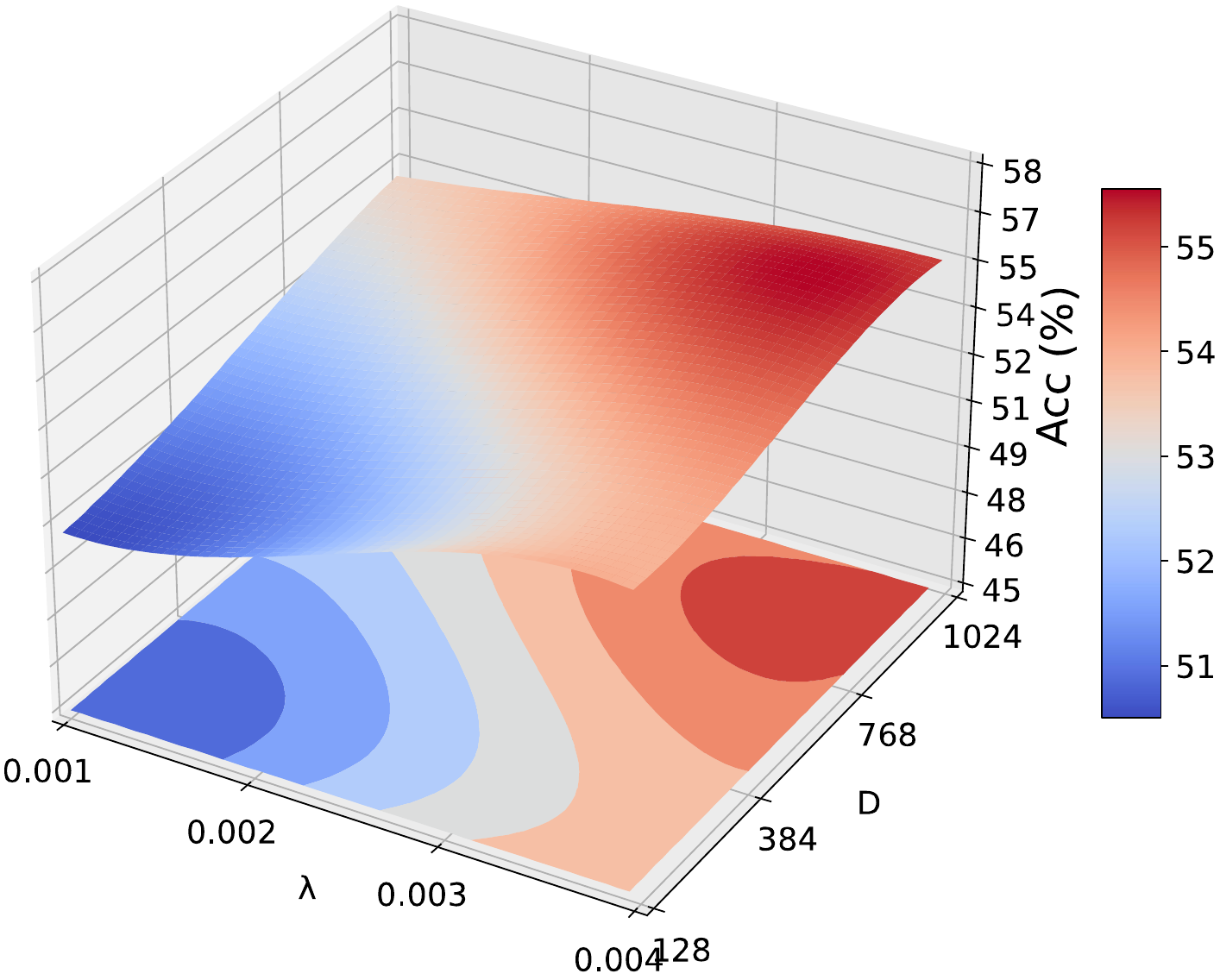}}
  \quad
  \caption{Hyperparameter sensitivity of the degree of MTMC  $\lambda$ and features dimensionality $D$.}
  \label{fig:ablation_3d}
\end{figure}

\textbf{Ablation study.} The only hyperparameter of MTMC is the coefficient $\lambda$ of the loss. To gain a deeper understanding of the correlation between the degree of maximum token manifold capacity and the dimensionality $D$ of the features, we conducted an ablation experiment on it, as shown in Figure~\ref{fig:ablation_3d}. It can be clearly observed that MTMC is not sensitive to hyperparameters and can uniformly enhance clustering accuracy. A more thought-provoking finding is that directly reducing $D$ to avoid dimensionality collapse is suboptimal. The reason is that each dimension of the manifold contributes to the representation, and a reduction in $D$ will directly lead to a loss of information. Even with MTMC, it is impossible to make the representation complete. An appropriate number of dimensions enriches the representation while using MTMC to prevent dimensionality collapse, which can maximize the model's performance enhancement.

\section{Hierarchical Analysis of Why MTMC is Effective in GCD}

We conduct a comprehensive analysis from multiple dimensions: 1) eigenvalue distribution and Frobenius norm, 2) estimation of embedded space distribution, 3) dimensional collapse, and 4) comparison with similar schemes, to understand the necessity and effectiveness of MTMC for GCD.

\begin{figure}[t]
\begin{center}
  \includegraphics[width=0.95\linewidth]{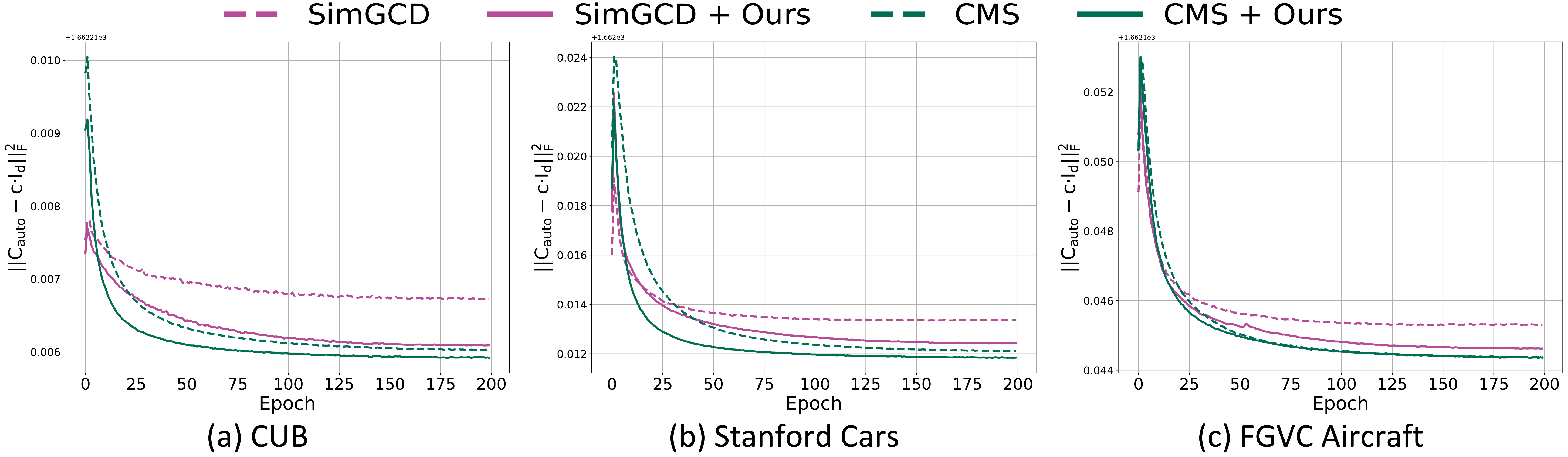}
\end{center}
\vspace{-0.2cm}
\caption{The Frobenius norm $\left\|\mathcal{A}-c \cdot I_d\right\|_F^2$ on three fine-grained benchmarks.}
\label{fig:norm}
\end{figure}

\subsection{MTMC Homogenizes Eigenvalue Distribution and Reduces Frobenius Norm} 
The autocorrelation matrix of the test sample class token manifold is denoted as $\mathcal{A}$. Given $\left\| \mathbf{\texttt{[cls]}}_i\right\|_2=1$ and $\mathcal{A} \geq 0$, it follows that $\sum_j \lambda_j = 1$ and $\forall_j \lambda_j \geq 0$~\cite{parkhi2015deep,liu2017sphereface,mettes2019hyperspherical}, where $\left\{\lambda_j\right\}$ are the eigenvalues of $\mathcal{A}$. Under ideal conditions, where $\mathcal{A} \rightarrow c \cdot I_d$ (maximum manifold capacity), the eigenvalue distribution of $\mathcal{A}$ becomes uniform, $\mathbf{z}$ uncorrelated~\cite{cogswell2015reducing}, full-rank~\cite{hua2021feature}, and isotropic~\cite{vershynin2018high}. $\mathcal{A}$ is linked to various representation characteristics. The Frobenius norm~\cite{ma1994frobenius,peng2016connections}, extensively studied in self-supervised learning methods~\cite{cogswell2015reducing,xiong2016regularizing,choi2019utilizing,zbontar2021barlow}, measures whether the representation depends on a few dimensions. A smaller Frobenius norm indicates a larger manifold capacity. We applied singular value decomposition (SVD)~\cite{golub1971singular} to the autocorrelation matrix of the feature embeddings, plotting the first 200 singular values in Figure~\ref{fig:eigenvalue_acc} and visualizing the Frobenius norm $\left|\mathcal{A} - c \cdot I_d\right|_F^2$ in Figure~\ref{fig:norm}. Compared to SimGCD and CMS, MTMC achieves a more uniform eigenvalue distribution and significantly reduces the Frobenius norm.

\subsection{MTMC Provides Accurate Distribution Estimation}

\input{tabs/k_estimation}
We present the gap between MTMC and SOTAs in estimating the number of clusters in Table~\ref{tab:kestimation}. By leveraging CMS, which requires no specific hyperparameters to estimate $K$, our optimization target becomes $\mathcal{L}_{\text{CMS}} + \mathcal{L}_{\text{MTMC}}$. Results demonstrate significant improvement with MTMC incorporated into the CMS framework, consistently enhancing class separation across various datasets. Notably, on the complex and diverse ImageNet100 dataset, our method achieves a 100\% correct estimation rate, reflecting the model's ability to discern fine-grained distinctions and align decision boundaries with the data's intrinsic structure. The improvement in estimating the number of clusters highlights the importance of representation completeness, enabling better capture of intra-class nuances and sharper inter-class separation.

\subsection{MTMC Unravels Dimensional Collapse.} 
We further explored the relationship between the accuracy and eigenvalues of GCD, respectively, and dimensional collapse, as shown in Figure \ref{fig:eigenvalue_acc} and our findings are as follows: (1) Feature Completeness and Clustering Accuracy: Complete features improve intra-class representations, which enhances clustering accuracy by providing richer, higher manifold capacity. (2) MTMC's Impact: MTMC increases manifold capacity, leading to higher singular values and more accurate clustering by better approximating the true distribution. (3) Dimension Collapse and Limitations of CMS/SimGCD: CMS and SimGCD operate in lower-dimensional spaces, limiting manifold capacity and causing incomplete representations~\cite{caron2020unsupervised,shi2023understanding}. Dimension collapse results in oversimplified models, while MTMC maximizes intra-class completeness for better decision boundaries. This breakdown highlights how MTMC addresses limitations in existing methods by optimizing the manifold capacity and the richness of intra-class representations, leading to improved model performance.

\begin{figure}[t]
\begin{center}
  \includegraphics[width=0.95\linewidth]{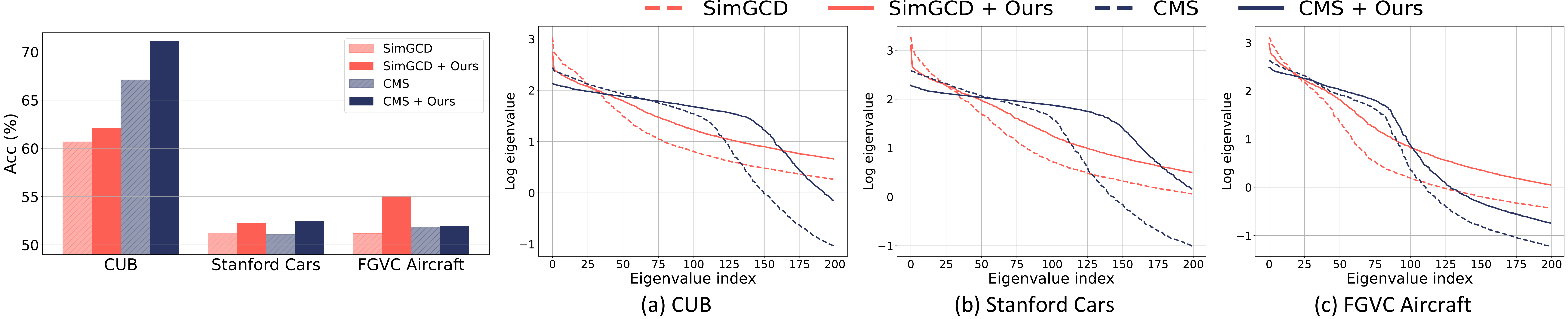}
\end{center}
\caption{MTMC effectively mitigates dimensional collapse by providing a more uniform eigenvalue distribution and improves the clustering accuracy.}
\label{fig:eigenvalue_acc}
\end{figure}

\subsection{Comparison with Isotropic Feature Distribution Schemes}

\input{tabs/gcd3}

From a motivation and self-supervised learning perspective based on Isotropic Feature Distribution, MTMC is similar. Therefore, we chose representatives from two schools, CorInfoMax \cite{corinfomax} and VICReg \cite{bardes2022vicreg}, as challengers. (1) \textbf{CorInfoMax}, from the mutual \textit{information maximization} approach, aims to maximize mutual information between features and their target distribution, enhancing feature representations by promoting decorrelation and information retention. (2) \textbf{VICReg} is a representative of the \textit{variance-based regularization} approach, promoting feature variance, invariance to augmentations, and low covariance to ensure a diverse feature space.

As shown in Table \ref{tab:gcd3}, while there is a certain degree of benefit for accuracy, it is minimal. In the context of GCD, VICReg and CorInfoMax suffer from key limitations that impact their performance. VICReg, while promoting variance and reducing covariance, does not explicitly focus on maximizing intra-class representation completeness, which is crucial for distinguishing fine-grained categories. This lack of emphasis on manifold capacity leads to less expressive class boundaries. CorInfoMax, on the other hand, primarily maximizes mutual information but does not explicitly prevent dimensional collapse or ensure richer intra-class representations. As a result, both methods struggle to capture the full complexity of the data's structure, limiting their effectiveness in accurately discovering novel categories compared to MTMC, which directly optimizes representation completeness and manifold capacity. Overall, compared to the two optimization directions VICReg and CorInfoMax, MTMC provides a smoother and more uniform convergence curve of feature values, consistent with the analysis and theoretical framework proposed in this paper, as shown in Figure \ref{fig:ablation_svd_lind}.

%% file: tabs/gcd.tex
\begin{table}[t]
    \caption{Experimental results on coarse- and fine-grained datasets, evaluated \textit{with} the $K$ for clustering.}
  \label{tab:gcd}
  \begin{center}
     \scalebox{0.70}{
 \renewcommand\tabcolsep{4.35pt}
    \begin{tabular}[b]{p{0.15\textwidth} ccc ccc ccc ccc ccc ccc ccc}
    \toprule
        \multirow{2}{*}{Method} &
        % \multirow{2.5}{*}{Known K} &
        \multicolumn{3}{c}{CIFAR100} &
        \multicolumn{3}{c}{ImageNet100} &
        \multicolumn{3}{c}{CUB} &
        \multicolumn{3}{c}{Stanford Cars} &
        \multicolumn{3}{c}{FGVC Aircraft} &
        \multicolumn{3}{c}{Herbarium 19} \\
        \cmidrule(lr){2-4} \cmidrule(lr){5-7} \cmidrule(lr){8-10}
        \cmidrule(lr){11-13} \cmidrule(lr){14-16} \cmidrule(lr){17-19}
        & All & Old & New & All & Old & New & All & Old & New
        & All & Old & New & All & Old & New & All & Old & New\\
    \midrule
    \multicolumn{19}{l}{\textit{ Clustering with the ground-truth number of classes $K$ given }} \\
    \midrule
    \small{Agglomerative}~\cite{ward1963hierarchical} & 56.9 & 56.6 & 57.5 & 73.1 & 77.9 & 70.6 & 37.0 & 36.2 & 37.3 & 12.5 & 14.1 & 11.7 & 15.5 & 12.9 & 16.9 & 14.4 & 14.6 & 14.4 \\
        RankStats+~\cite{han2020automatically} & 58.2 & 77.6 & 19.3 & 37.1 & 61.6 & 24.8 & 33.3 & 51.6 & 24.2 & 28.3 & 61.8 & 12.1 & 26.9 & 36.4 & 22.2 & 27.9 & 55.8 & 12.8 \\
        UNO+~\cite{fini2021unified} & 69.5 & 80.6 & 47.2 & 70.3 & 95.0 & 57.9 & 35.1 & 49.0 & 28.1 & 35.5 & 70.5 & 18.6 & 40.3 & 56.4 & 32.2 & 28.3 & 53.7 & 14.7 \\
        ORCA~\cite{cao2022openworld} & 69.0 & 77.4 & 52.0 & 73.5 & 92.6 & 63.9 & 35.3 & 45.6 & 30.2 & 23.5 & 50.1 & 10.7 & 22.0 & 31.8 & 17.1 & 20.9 & 30.9 & 15.5\\
        GCD~\cite{vaze2022generalized} & 73.0 & 76.2 & 66.5 & 74.1 & 89.8 & 66.3 & 51.3 & 56.6 & 48.7 & 39.0 & 57.6 & 29.9 & 45.0 & 41.1 & 46.9 & 35.4 & 51.0 & 27.0\\
        ProtoGCD~\cite{ma2025protogcd} & 81.9 & 82.9 & 80.0 & 84.0 & 92.2 & 79.9 & 63.2 & 68.5 & 60.5 & 53.8 & 73.7 & 44.2 & 56.8 & 62.5 & 53.9 & 44.5 & 59.4 & 36.5 \\
        PrCAL~\cite{zhang2023promptcal}& 81.2 & 84.2 & 75.3 & 83.1 & 92.7 & 78.3 & 62.9 & 64.4 & 62.1 & 50.2 & 70.1 & 40.6 & 52.2 & 52.2 & 52.3 & 37.0 & 52.0 & 28.9\\
        ActiveGCD~\cite{ma2024active} & 71.3 & 75.7 & 66.8 & 83.3 & 90.2 & 76.5 & 66.6 & 66.5 & 66.7 & 48.4 & 57.7 & 39.3 & 53.7 & 51.5 & 56.0 & - & - & - \\
        PIM~\cite{chiaroni2023parametric} & 78.3 & 84.2 & 66.5 & 83.1 & 95.3 & 77.0 & 62.7 & 75.7 & 56.2 & 43.1 & 66.9 & 31.6 & - & - & - & 42.3 & 56.1 & 34.8\\
    \midrule
        SelEx~\cite{wen2023parametric} & 80.0 & 84.8 & 70.4 & 82.3 & 93.9 & 76.5 & 78.7 & \textbf{81.3} & 77.5 & 55.9 & 76.9 & 45.8 & 60.8 & \textbf{70.3} & 56.2 & 36.2 & 46.0 & 30.9 \\
         \rowcolor{gray!15}\quad+ Ours &  80.7 &  84.3 &  72.1 &  82.8 &  94.1 &  77.8 &  \textbf{80.6} &  81.0 &  \textbf{80.4} &  57.0 &  \textbf{77.3} &  47.2 &  \textbf{61.8} &  68.2 &  \textbf{59.2} &  36.8 &  47.5 &  31.0 \\
          & \ccol \textbf{\textcolor{ForestGreen}{+0.7}} & \ccol \textcolor{red!50!white}{-0.5} & \ccol \textbf{\textcolor{ForestGreen}{+1.7}} & \ccol \textbf{\textcolor{ForestGreen}{+0.5}} & \ccol \textbf{\textcolor{ForestGreen}{+0.2}} & \ccol \textbf{\textcolor{ForestGreen}{+1.3}} & \ccol \textbf{\textcolor{ForestGreen}{+1.9}} & \ccol \textcolor{red!50!white}{-0.3} & \ccol \textbf{\textcolor{ForestGreen}{+2.9}} & \ccol \textbf{\textcolor{ForestGreen}{+1.1}} & \ccol \textbf{\textcolor{ForestGreen}{+0.4}} & \ccol \textbf{\textcolor{ForestGreen}{+1.4}} & \ccol \textbf{\textcolor{ForestGreen}{+1.0}} & \ccol \textcolor{red!50!white}{-2.1} & \ccol \textbf{\textcolor{ForestGreen}{+3.0}} & \ccol \textbf{\textcolor{ForestGreen}{+0.6}} & \ccol \textbf{\textcolor{ForestGreen}{+1.5}} & \ccol \textbf{\textcolor{ForestGreen}{+0.1}} \\
   % \cline{2-19}     
        SimGCD~\cite{wen2023parametric} & 80.1 & 81.5 & 77.2 & 83.3 & 92.1 & 78.9 & 60.7 & 65.6 & 57.7 & 51.2 & 69.4 & 42.4 & 54.0 & 58.8 & 51.5 & 44.7 & 57.4 & 37.9 \\
         \rowcolor{gray!15} \quad+ Ours &  80.2 &  81.5 &  \textbf{77.5} &  \textbf{86.7} &  93.1 &  \textbf{83.6} &  62.1 &  65.8 &  60.3 &  52.3 &  70.0 &  43.7 &  55.1 &  58.9 &  53.1 &  \textbf{45.6} &  57.8 &  \textbf{39.0} \\
          & \ccol \textbf{\textcolor{ForestGreen}{+0.1}} & \ccol \textbf{\textcolor{ForestGreen}{+0.0}} & \ccol \textbf{\textcolor{ForestGreen}{+0.3}} & \ccol \textbf{\textcolor{ForestGreen}{+3.4}} & \ccol \textbf{\textcolor{ForestGreen}{+1.0}} & \ccol \textbf{\textcolor{ForestGreen}{+4.7}} & \ccol \textbf{\textcolor{ForestGreen}{+1.4}} & \ccol \textbf{\textcolor{ForestGreen}{+0.2}} & \ccol \textbf{\textcolor{ForestGreen}{+2.6}} & \ccol \textbf{\textcolor{ForestGreen}{+1.1}} & \ccol \textbf{\textcolor{ForestGreen}{+0.6}} & \ccol \textbf{\textcolor{ForestGreen}{+1.3}} & \ccol \textbf{\textcolor{ForestGreen}{+1.1}} & \ccol \textbf{\textcolor{ForestGreen}{+0.1}} & \ccol \textbf{\textcolor{ForestGreen}{+1.6}} & \ccol \textbf{\textcolor{ForestGreen}{+0.9}} & \ccol \textbf{\textcolor{ForestGreen}{+0.4}} & \ccol \textbf{\textcolor{ForestGreen}{+1.1}} \\
        
       CMS~\cite{choi2024contrastive}$\dagger$ & 79.5 & 85.4 & 67.7 & 83.0 & \textbf{95.6} & 76.6 & 67.1 & 74.9 & 63.2 & 56.7 & 76.8 & 37.5 & 53.6 & 60.3 & 47.0 & 36.5 & 55.4 & 26.4 \\
        \rowcolor{gray!15} \quad+ Ours &  79.0 &  \textbf{85.5} &  66.1 &  84.8 &  \textbf{95.6} &  79.5 &  71.1 &  74.1 &  66.9 &  57.4 &  79.4 &  36.2 &  55.7 &  63.7 &  47.9 &  36.3 &  56.5 &  25.4 \\
         & \ccol \textcolor{red!50!white}{-0.5} & \ccol \textbf{\textcolor{ForestGreen}{+0.1}} & \ccol \textcolor{red!50!white}{-1.6} & \ccol \textbf{\textcolor{ForestGreen}{+1.8}} & \ccol \textbf{\textcolor{ForestGreen}{+0.0}} & \ccol \textbf{\textcolor{ForestGreen}{+2.9}} & \ccol \textbf{\textcolor{ForestGreen}{+4.0}} & \ccol \textcolor{red!50!white}{-0.8} & \ccol \textbf{\textcolor{ForestGreen}{+3.7}} & \ccol \textbf{\textcolor{ForestGreen}{+0.7}} & \ccol \textbf{\textcolor{ForestGreen}{+2.6}} & \ccol \textcolor{red!50!white}{-1.3} & \ccol \textbf{\textcolor{ForestGreen}{+2.1}} & \ccol \textbf{\textcolor{ForestGreen}{+3.4}} & \ccol \textbf{\textcolor{ForestGreen}{+0.9}} & \ccol \textcolor{red!50!white}{-0.2} & \ccol \textbf{\textcolor{ForestGreen}{+1.1}} & \ccol \textcolor{red!50!white}{-1.0} \\

        SPTNet~\cite{choi2024contrastive} & 81.3 & 84.3 & 75.6 & 85.4 & 93.2 & 81.4 & 62.0 & 69.2 & 56.0 & 56.2 & 70.3 & 46.6 & 51.6 & 60.7 & 45.9 & 43.4 &  58.7 &  35.2 \\
         \rowcolor{gray!15} \quad+ Ours &  \textbf{82.1} &  84.8 &  76.2 &  85.4 &  93.4 &  81.3 &  63.3 &  70.7 &  59.6 &  \textbf{58.8} &  75.4 &  \textbf{50.8} &  54.7 &  65.3 &  48.5 & 44.2 &  \textbf{58.9} &  36.3 \\
         & \ccol \textbf{\textcolor{ForestGreen}{+0.8}} & \ccol \textbf{\textcolor{ForestGreen}{+0.5}} & \ccol \textbf{\textcolor{ForestGreen}{+0.6}} & \ccol \textbf{\textcolor{ForestGreen}{+0.0}} & \ccol \textbf{\textcolor{ForestGreen}{+0.2}} & \ccol \textcolor{red!50!white}{-0.1} & \ccol \textbf{\textcolor{ForestGreen}{+1.3}} & \ccol \textbf{\textcolor{ForestGreen}{+1.5}} & \ccol \textbf{\textcolor{ForestGreen}{+3.6}} & \ccol \textbf{\textcolor{ForestGreen}{+2.6}} & \ccol \textbf{\textcolor{ForestGreen}{+5.1}} & \ccol \textbf{\textcolor{ForestGreen}{+4.2}} & \ccol \textbf{\textcolor{ForestGreen}{+3.1}} & \ccol \textbf{\textcolor{ForestGreen}{+4.5}} & \ccol \textbf{\textcolor{ForestGreen}{+2.6}} & \ccol \textbf{\textcolor{ForestGreen}{+0.8}} & \ccol \textbf{\textcolor{ForestGreen}{+0.1}} & \ccol \textbf{\textcolor{ForestGreen}{+1.1}} \\
\midrule
           \quad \textbf{Avg.} $\triangle$ & \ccol \textbf{\textcolor{ForestGreen}{+0.3}} & \ccol \textbf{\textcolor{ForestGreen}{+0.1}} & \ccol \textbf{\textcolor{ForestGreen}{+0.3}} & \ccol \textbf{\textcolor{ForestGreen}{+1.4}} & \ccol \textbf{\textcolor{ForestGreen}{+0.4}} & \ccol \textbf{\textcolor{ForestGreen}{+2.2}} & \ccol \textbf{\textcolor{ForestGreen}{+2.2}} & \ccol \textbf{\textcolor{ForestGreen}{+0.2}} & \ccol \textbf{\textcolor{ForestGreen}{+3.2}} & \ccol \textbf{\textcolor{ForestGreen}{+1.4}} & \ccol \textbf{\textcolor{ForestGreen}{+2.2}} & \ccol \textbf{\textcolor{ForestGreen}{+1.1}} & \ccol \textbf{\textcolor{ForestGreen}{+1.8}} & \ccol \textbf{\textcolor{ForestGreen}{+1.5}} & \ccol \textbf{\textcolor{ForestGreen}{+2.0}} & \ccol \textbf{\textcolor{ForestGreen}{+0.5}} & \ccol \textbf{\textcolor{ForestGreen}{+0.8}} & \ccol \textbf{\textcolor{ForestGreen}{+0.3}} \\

    \bottomrule
    \end{tabular}}
\end{center}
\end{table}

%% file: tabs/gcd2.tex
\begin{table}[t]
    \caption{GCD Accuracy on coarse- and fine-grained datasets, evaluated \textit{without} the $K$ for clustering.}
  \label{tab:gcd2}
  \begin{center}
     \scalebox{0.665}{
 \renewcommand\tabcolsep{4.35pt}
    \begin{tabular}[b]{p{0.15\textwidth} ccc ccc ccc ccc ccc ccc}
    \toprule
        \multirow{2}{*}{Method} &
        % \multirow{2.5}{*}{Known K} &
        \multicolumn{3}{c}{CIFAR100} &
        \multicolumn{3}{c}{ImageNet100} &
        \multicolumn{3}{c}{CUB} &
        \multicolumn{3}{c}{Stanford Cars} &
        \multicolumn{3}{c}{FGVC Aircraft} &
        \multicolumn{3}{c}{Herbarium 19} \\
        \cmidrule(lr){2-4} \cmidrule(lr){5-7} \cmidrule(lr){8-10}
        \cmidrule(lr){11-13} \cmidrule(lr){14-16} \cmidrule(lr){17-19}
        & All & Old & New & All & Old & New & All & Old & New
        & All & Old & New & All & Old & New & All & Old & New\\
    \midrule
    % [5pt]
    \multicolumn{19}{l}{\textit{ Clustering without the ground-truth number of classes $K$ given }} \\
    \midrule
     \small{Agglomerative}~\cite{ward1963hierarchical} & 56.9 & 56.6 & 57.5 & 72.2 & 77.8 & 69.4 & 35.7 & 33.3 & 36.9 & 10.8 & 10.6 & 10.9 & 14.1 & 10.3 & 16.0 & 13.9 & 13.6 & 14.1 \\
        GCD~\cite{vaze2022generalized} & 70.8 & 77.6 & 57.0 & 77.9 & 91.1 & 71.3 & 51.1 & 56.4 & 48.4 
                                & 39.1 & 58.6 & 29.7 & - & - & - & 37.2 & 51.7 & 29.4 \\
        GPC~\cite{zhao2023learning} & 75.4 & 84.6 & 60.1 & 75.3 & 93.4 & 66.7 & 52.0 & 55.5 & 47.5 &
                                    38.2 & 58.9 & 27.4 & 43.3 & 40.7 & 44.8 & 36.5 & 51.7 & 27.9 \\
        PIM~\cite{chiaroni2023parametric} & 75.6 & 81.6 & 63.6 & 83.0 & 95.3 & 76.9 & 62.0 &                                   \textbf{75.7} & 55.1 & 42.4 & 65.3 & 31.3 & - & - & - & \textbf{42.0} & 55.5 & \textbf{34.7} \\
    \midrule
        % SimGCD~\cite{wen2023parametric} & 80.1 & 81.2 & \textbf{77.8} & 83.0 & 93.1 & 77.9 & 60.3 & 65.6 & 57.7 & 53.8 & 71.9 & 45.0 & 54.2 & 59.1 & 51.8 & \textbf{44.0} & \textbf{58.0} & \textbf{36.4} \\
        % \ccol \quad$+$  Ours & \ccol \textbf{79.6} & \ccol 83.2 & \ccol \textbf{72.3} & \ccol 81.3 & \ccol \textbf{95.6} & \ccol 74.2 & \ccol \textbf{64.4} & \ccol 68.2 & \ccol \textbf{62.4} & \ccol \textbf{51.7} & \ccol \textbf{68.9} & \ccol \textbf{43.4} & \ccol \textbf{55.2} & \ccol \textbf{60.6} & \ccol \textbf{52.4} & \ccol 37.4 & \ccol \textbf{56.5} & \ccol 27.1 \\
        CMS~\cite{choi2024contrastive} & 77.8 & 84.0 & 65.3 & 83.4 & 95.6 & 77.3 & 66.2 & 69.7 & 64.4 & 51.8 & \textbf{72.9} & 31.3 & 52.3 & 58.9 & 45.8 & 38.5 & \textbf{57.3} & 28.4 \\
         \rowcolor{gray!15} \quad$+$ Ours &  \textbf{79.5} &  \textbf{84.7} &  \textbf{69.1} &  \textbf{84.3} &  \textbf{95.7} &  \textbf{78.8} &  \textbf{68.7} &  74.1 &  \textbf{66.0} &  \textbf{52.5} &  72.7 &  \textbf{32.9} &  \textbf{53.4} &  \textbf{60.1} &  \textbf{46.7} &  38.0 &  56.9 &  27.9 \\
       \quad \textbf{Avg.} $\triangle$ & \ccol \textbf{\textcolor{ForestGreen}{$+$1.7}} & \ccol \textbf{\textcolor{ForestGreen}{$+$0.7}} & \ccol \textbf{\textcolor{ForestGreen}{$+$3.8}} & \ccol \textbf{\textcolor{ForestGreen}{$+$0.9}} & \ccol \textbf{\textcolor{ForestGreen}{$+$0.1}} & \ccol \textbf{\textcolor{ForestGreen}{$+$1.5}} & \ccol \textbf{\textcolor{ForestGreen}{$+$2.5}} & \ccol \textbf{\textcolor{ForestGreen}{$+$4.4}} & \ccol \textbf{\textcolor{ForestGreen}{$+$1.6}} & \ccol \textbf{\textcolor{ForestGreen}{$+$0.7}} & \ccol \textcolor{red!50!white}{-0.2} & \ccol \textbf{\textcolor{ForestGreen}{$+$1.6}} & \ccol \textbf{\textcolor{ForestGreen}{$+$1.1}} & \ccol \textbf{\textcolor{ForestGreen}{$+$1.2}} & \ccol \textbf{\textcolor{ForestGreen}{$+$0.9}} & \ccol \textcolor{red!50!white}{-0.5} & \ccol \textcolor{red!50!white}{-0.4} & \ccol \textcolor{red!50!white}{-0.5} \\
    \bottomrule
    \end{tabular}}
\end{center}
\end{table}

%% file: tabs/k_estimation.tex
\begin{wraptable}{r}{7cm}
\vspace{-0.4cm}
  \centering
    \caption{Estimated number and error rate of $K$.}
  \label{tab:kestimation}
       \scalebox{0.45}
 {
     \begin{tabular}[b]{l cc cc cc cc cc}
    \toprule
        \multirow{2.5}{*}{Method} &
        \multicolumn{2}{c}{CIFAR100} &
        \multicolumn{2}{c}{ImageNet100} &
        \multicolumn{2}{c}{CUB} &
        \multicolumn{2}{c}{Stanford Cars} &
        \multicolumn{2}{c}{FGVC Aircraft}  \\
        \cmidrule(lr){2-3} \cmidrule(lr){4-5} \cmidrule(lr){6-7} \cmidrule(lr){8-9} \cmidrule(lr){10-11}
         & K & Err(\%) &  K & Err(\%) & K & Err(\%)  & K & Err(\%) &  K & Err(\%) \\
    \midrule
        Ground truth & 100 & - & 100 & - & 200 & - & 196 & - & 100 & - \\
    \midrule
        GCD~\cite{vaze2022generalized} & 100 & 0 & 109 & 9  & 231 & 15.5 & 230 & 17.3 & - & -   \\
        DCCL~\cite{pu2023dynamic} & 146 & 46 & 129 & 29& 172 & 9 & 192 & 0.02 & - & -  \\
        PIM~\cite{chiaroni2023parametric} & 95 & 5 & 102 & 2 & 227 & 13.5 & 169 & 13.8 & - & -  \\
        GPC~\cite{zhao2023learning} & 100 & 0 & 103 & 3 & 212 & 6 & 201 & 0.03 & - & -  \\
    \midrule
        % SimGCD~\cite{wen2023parametric} & 100 & 0 & 103 & 3 & 212 & 6 & 201 & 0.03 & - & - & - & - \\
        % \ccol \quad +Ours & \ccol 95 & \ccol 5 & \ccol 116 & \ccol 16 & \ccol 168 & \ccol 16 & \ccol 156 & \ccol 20.4 & \ccol 90 & \ccol 10 & \ccol 622 & \ccol 8.9 \\
        CMS~\cite{choi2024contrastive}$\dagger$ & 94 & 6 & 98 & 2 & 176 & 12 & 149 & 23.9 \\
        \rowcolor{gray!15} \ccol \quad + Ours & \ccol 96 & \ccol 4 & \ccol 100 & \ccol 0 & \ccol 180 & \ccol 10 & \ccol 159 & \ccol 18.9 & \ccol 89 & \ccol 11  \\
    \bottomrule
    \end{tabular}%
}
% \vspace{-0.2cm}
  % \caption{Ablation study on PACS dataset (P@50).}
  % \label{tab:ablation}
  \vspace{-0.2cm}
\end{wraptable}

%% file: tabs/gcd3.tex
\begin{wraptable}{r}{7cm}
    \vspace{-0.4cm}
    \caption{Comparison on accuracy in GCD with representative isotropic feature distribution schemes.}
    \vspace{-0.2cm}
  \label{tab:gcd3}
  \begin{center}
     \scalebox{0.5}
 {
 \renewcommand\tabcolsep{4.35pt}
    \begin{tabular}[b]{p{0.17\textwidth} ccc ccc ccc ccc}
    \toprule
        \multirow{2.5}{*}{Method} &
        % \multirow{2.5}{*}{Known K} &
        \multicolumn{3}{c}{CUB} &
        \multicolumn{3}{c}{Stanford Cars} &
        \multicolumn{3}{c}{FGVC Aircraft} &
        \multicolumn{3}{c}{\textbf{ Average}} \\
        \cmidrule(lr){2-4} \cmidrule(lr){5-7} \cmidrule(lr){8-10}
        \cmidrule(lr){11-13} 
        & All & Old & New & All & Old & New & All & Old & New
        & All & Old & New \\
    \midrule
    % [5pt]
    \midrule
        SimGCD~\cite{wen2023parametric} & 60.7 & 65.6 & 57.7 & 51.2 & 69.4 & 42.4 & 54.0 & 58.8 & 51.5 & 55.3 & 64.6 & 50.5 \\
        \quad+CorInfoMax & 60.7 & 64.8 & 58.6 & 50.0 & 67.4 & 41.6 & 54.4 & \textbf{59.0} & 52.1 & 55.0 & 63.7 & 50.8 \\
        \quad+VICReg & 61.1 & \textbf{66.0} & 58.1 & 52.0 & 68.6 & \textbf{44.1} & 54.6 & 56.2 & \textbf{53.8} & 55.9 & 63.6 & 52.0 \\
         \rowcolor{gray!15} \quad+Ours &  \textbf{62.1} &  65.8 &  \textbf{60.3} &  \textbf{52.3} & \textbf{70.0} &  43.7 &  \textbf{55.1} &  58.9 &  53.1 &  \textbf{56.5} &  \textbf{64.9} &  \textbf{52.4} \\
         \midrule
         CMS~\cite{choi2024contrastive} & 67.1 & 74.9 & 63.2 & 56.7 & 76.8 & 37.5 & 53.6 & 60.3 & 47.0 & 59.1 & 70.7 & 49.2 \\
        \quad+CorInfoMax & 65.7 & 76.4 & 58.7 & 55.8 & 73.1 & 39.2 & 52.4 & 61.9 & 42.8 & 58.0 & 70.5 & 46.9 \\
        \quad+VICReg & 68.3 & \textbf{78.1} & 55.0 &\textbf{57.8} & 76.7 & \textbf{39.7} & 55.2 & \textbf{65.2} & 45.1 & 60.4 & \textbf{73.3} & 46.6 \\
        \rowcolor{gray!15} \quad+Ours &  \textbf{71.1} &  74.1 &  \textbf{66.9} &  57.4 &  \textbf{79.4} &  36.2 &  \textbf{55.7} &  63.7 &  \textbf{47.9} &  \textbf{61.4} &  72.4 &  \textbf{50.3} \\
    \bottomrule
    \end{tabular}%
}
\vspace{-0.6cm}
\end{center}
\end{wraptable}

%% file: secs/6_conclusion.tex
\section{Conclusion}
\label{sec:con}

% The paper introduces a straightforward approach to enhancing Generalized Category Discovery by Maximum Token Manifold Capacity. Our method counters the traditional focus on compact clusters, which can lead to low manifold capacity and incomplete representations. Emphasizing the integrity of intra-class representations, MTMC leverages the nuclear norm to ensure manifolds are both compact and informative. Through extensive experiments, we demonstrated that our proposal significantly improves clustering accuracy and the estimation of category numbers. Theoretically, MTMC prevents dimensional collapse, leading to a more uniform eigenvalue distribution and higher entropy, indicative of richer representations. Our method's effectiveness in GCD lies in its promotion of complete and non-collapsed representations, paving the way for more robust and adaptable machine learning models in open-world scenarios.

We introduces Maximum Token Manifold Capacity, a simple yet powerful approach for enhancing Generalized Category Discovery. By focusing on maximizing the manifold capacity of class tokens, MTMC prevents dimensional collapse, ensuring that intra-class representations are both complete and rich. This approach effectively addresses the limitations of traditional GCD methods, which often sacrifice representation quality for compact clustering. Our theoretical analysis and experiments show that MTMC significantly improves clustering accuracy, category number estimation, and inter-class separability, without introducing excessive computational complexity. Through extensive evaluations on both coarse- and fine-grained datasets, we demonstrate that MTMC enhances performance even on challenging benchmarks, making it a critical tool for open-world learning. By promoting comprehensive, non-collapsed representations, MTMC unlocks the model's full potential for more adaptable and robust machine learning models in real-world scenarios.

%% file: secs/X_appendix.tex
\clearpage

\section{Details of optimization objective of GCD}
\label{sec:gcd}

The existing GCD proposals are all proposed for compact clustering. Summarizing the optimization objectives of mainstream schemes GCD~\cite{vaze2022generalized}, CMS~\cite{choi2024contrastive} and SimGCD~\cite{wen2023parametric}, it can be observed that they are based on contrastive learning or prototype learning to significantly reduce the distance between potentially similar samples in the feature space.

\subsection{GCD}
The pioneering work~\cite{vaze2022generalized} divided the mini-batch $\mathcal{B}$ into labelled $\mathcal{B}^{l}$ and unlabeled $\mathcal{B}^{u}$, using supervised~\cite{khosla2020supervised} contrastive learning $\mathcal{L}_{\text{GCD}}^l = -\frac{1}{|\mathcal{B}^{l}|}\sum_{i\in \mathcal{B}^{l}}\frac{1}{|\mathcal{B}^{l}(i)|}\sum_{j\in\mathcal{B}^{l}(i)} \log\frac{\exp(\mathbf{z}_i^\top \mathbf{z}_j^\prime/\tau)}{\sum_{n\neq i} \exp(\mathbf{z}_i^\top \mathbf{z}_n^\prime/\tau)}$, and self-supervised~\cite{chen2020simple} contrastive learning $\mathcal{L}_{\text{GCD}}^u = -\frac{1}{|\mathcal{B}|}\sum_{i\in \mathcal{B}} \log\frac{\exp(\mathbf{z}_i^\top \mathbf{z}_i^\prime/\tau)}{\sum_{n\neq i} \exp(\mathbf{z}_i^\top \mathbf{z}_n^\prime/\tau)}$ and balancing them using coefficients $\lambda$: $ \mathcal{L}_{\text{GCD}} = (1-\lambda)\mathcal{L}_{\text{GCD}}^u + \lambda \mathcal{L}_{\text{GCD}}^l $, where $\mathcal{B}^{l}(i)$ represents the collection of samples with the same label as $i$. The $\mathbf{z}$ and $\mathbf{z}^\prime$ are augmented from two different views, and the $\tau$ is the temperature.

\subsection{CMS}
CMS~\cite{choi2024contrastive} and GCD adopt similar supervised and self-supervised contrastive learning. The difference is that CMS introduced mean-shift into unsupervised learning. For the $i$-th sample, CMS collects the feature set $\mathcal{V}=\left\{\mathbf{z}_i\right\}_{i=1}^N$ of training samples and calculates the k-nearest neighbours $\mathcal{N}\left(\mathbf{z}_i\right)=\left\{\mathbf{z}_i\right\} \cup \operatorname{argmax}_{\mathbf{z}_j \in \mathcal{V}}^k \mathbf{z}_i \cdot \mathbf{z}_j$, where $\operatorname{argmax}_{s \in \mathcal{S}}^k(\cdot)$ returns a subset of the top-$k$ items. By aggregating neighbor embeddings with weight kernel $\varphi(\cdot)$, it obtains the new embedded representation of samples after mean-shift: $\hat{\mathbf{z}_{i}}=\frac{\sum_{\mathbf{z}_j \in \mathcal{N}\left(\mathbf{z}_i\right)} \varphi\left(\mathbf{z}_j-\mathbf{z}_i\right) \mathbf{z}_j}{\left\|\sum_{\mathbf{z}_j \in \mathcal{N}\left(\mathbf{z}_i\right)} \varphi\left(\mathbf{z}_j-\mathbf{z}_i\right) \mathbf{z}_j\right\|}$. $\mathcal{L}_{\text{CMS}}$ and $\mathcal{L}_{\text{GCD}}$ are formally approximate.

\subsection{SimGCD} 
SimGCD~\cite{wen2023parametric} constructs a prototype classifier $\mathbf{C}=$ $\left\{\mathbf{c}_1, \cdots, \mathbf{c}_{K_{\text {known}}+K_{\text {novel}}}\right\}$ for both known and unknown classes. It obtains the posterior probability $\mathbf{p}_i^{(k)}=\frac{\exp \left(\mathbf{h}_i^{\top} \mathbf{c}_k\right) / \tau}{\sum_{k^{\prime}} \exp \left(\mathbf{h}_i^{\top} \mathbf{c}_k^{\prime}\right) / \tau}$ in a similar way to FixMatch and uses cross-entropy loss $\mathcal{L}_{\text{SimGCD}}^l=\frac{1}{\left| \mathcal{B}^l\right|} \sum_{i \in  \mathcal{B}^l} \ell\left(y_i, \mathbf{p}_i\right)$ on labeled samples. Self-distillation and entropy regularization $\mathcal{L}_{\text {SimGCD}}^u=\frac{1}{|\mathcal{B}|} \ell\left(\mathbf{p}_i^{\prime}, \mathbf{p}_i\right)-\lambda_e H(\frac{1}{2|\mathcal{B}|} \sum_{i \in \mathcal{B}}\left(\mathbf{p}_i+\mathbf{p}_i^{\prime}\right))$ are performed using augmented samples with probability $\mathbf{p}_i^{\prime}$.

\section{Proofs of Theorem}
\label{sec:theorem}

\begin{lemma}
\label{lem:entropy}
Given non-negative values \( p_i \) such that \( \sum_{i=1}^n p_i = 1 \), the entropy function \( H(p_1, \ldots, p_n) = -\sum_{i=1}^n p_i \log p_i \) is strictly concave. Furthermore, it is upper-bounded by \( \log n \), as demonstrated by the inequality,
% \vspace{-0.2cm}
\begin{equation}
\label{eq:property_entropy}
\log{n} = H(1/n,...,1/n) \ge H(p_{1},...,p_{n}) \ge 0.
% \vspace{-0.2cm}
\end{equation}
\end{lemma}
\begin{proof}
Refer to Section D.1 in~\cite{marshall1979inequalities}.
\end{proof}

\begin{lemma}
\label{lem:kld_gaussian}
The Kullback-Leibler (KL) divergence between two zero-mean, \( d \)-dimensional multivariate Gaussian distributions can be formulated as follows,
% \vspace{-0.2cm}
\begin{equation}
\begin{split}
&D_{\mathrm{KL}}( \mathcal{N} (0, \bm{\Sigma}_1) \Vert \mathcal{N} (0, \bm{\Sigma}_2) )\\
&= \frac{1}{2}\left[\text{tr}(\bm{\Sigma}_2^{-1}\bm{\Sigma}_1)-d +\log\frac{|\bm{\Sigma}_2|}{|\bm{\Sigma}_1|} \right].
% \vspace{-0.2cm}
\end{split}
\end{equation}
\end{lemma}
\begin{proof}
Refer to Section 9 in~\cite{duchi2007derivations}.
\end{proof}

\begin{theorem}
For a given \texttt{[cls]} autocorrelation $\mathcal{A} =\mathbf{CLS}^{\top} \mathbf{CLS} / N  \in \mathbb{R}^{d \times d}$ of rank $k$ ($\leq d$),
\begin{equation}
    \log \left(\operatorname{rank}\left(\mathcal{A}\right)\right) \geq \hat{H}\left(\mathcal{A}\right)
\end{equation}
where equality holds if the eigenvalues of $\mathcal{A}$ are uniformly distributed with $\forall_{j=1}^k \lambda_j=1 / k$ and $\forall_{j=k+1}^d \lambda_j=0$.
\end{theorem}

\begin{proof}
\begin{align}
\text{log}(\text{rank}(\mathcal{A})) &= \text{log}(k)\\[-2pt]
\label{eq:inequlity_rank}
&\ge H(\lambda_{1},...,\lambda_{k}) \text{ (by Lemma~\ref{lem:entropy})}\\[-2pt]
&=-\sum_{j=1}^k \lambda_{j} \log{\lambda_{j}}\\[-2pt]
\label{eq:equlity_rank}
&=-\sum_{j=1}^d \lambda_{j} \log{\lambda_{j}}\\[-2pt]
&=\hat{H}(\mathcal{A}).
\end{align}
According to Lemma~\ref{lem:entropy}, the inequality~\eqref{eq:inequlity_rank} attains equality if and only if \( \lambda_j = \frac{1}{k} \) for all \( j = 1, 2, \ldots, k \). Equation~\eqref{eq:equlity_rank} adheres to the convention that \( 0 \log 0 = 0 \), as per the definition in~\cite{thomas2006elements}.
\end{proof}

More details about the definition of  $\lambda$=1. Suppose we have a set of $n$ normalized vectors $\mathbf{v}_1, \mathbf{v}_2, \ldots, \mathbf{v}_n$, where the second-order norm (or length) of each vector is 1, that is, $\|\mathbf{v}_i\| = 1$ for all $i$. The autocorrelation matrix $\mathbf{A}$ of these vectors is defined as:
\begin{equation}
\mathbf{A} = \frac{1}{n} \sum_{i = 1}^n \mathbf{v}_i \mathbf{v}_i^T
\end{equation}

Here, $\mathbf{v}_i \mathbf{v}_i^T$ is the outer product of the vector $\mathbf{v}_i$ with itself, which is a rank-1 matrix. The autocorrelation matrix $\mathbf{A}$ is the average of these outer product matrices.

Next, we need to find the eigenvalues of $\mathbf{A}$. Since $\mathbf{v}_i$ is normalized, $\mathbf{v}_i^T \mathbf{v}_i = 1$. This means that each $\mathbf{v}_i$ is an eigenvector of $\mathbf{A}$ with the corresponding eigenvalue of $\frac{1}{n}$. This is because:

\begin{equation}
\mathbf{A} \mathbf{v}_i = \frac{1}{n} \left( \sum_{j = 1}^n \mathbf{v}_j \mathbf{v}_j^T \right) \mathbf{v}_i = \frac{1}{n} \sum_{j = 1}^n \mathbf{v}_j (\mathbf{v}_j^T \mathbf{v}_i) = \frac{1}{n} \mathbf{v}_i (\mathbf{v}_i^T \mathbf{v}_i) = \frac{1}{n} \mathbf{v}_i \cdot 1 = \frac{1}{n} \mathbf{v}_i
\end{equation}

So, the eigenvalue corresponding to each $\mathbf{v}_i$ is $\frac{1}{n}$. Since $\mathbf{A}$ is a rank-$n$ matrix (assuming the vectors $\mathbf{v}_1, \mathbf{v}_2, \ldots, \mathbf{v}_n$ are linearly independent), it has $n$ eigenvalues. We already know that $n$ of the $n$ eigenvalues are $\frac{1}{n}$. Therefore, the sum of all the eigenvalues of $\mathbf{A}$ is:

\begin{equation}
\text{sum of eigenvalues} = \frac{1}{n} + \frac{1}{n} + \cdots + \frac{1}{n} = n \cdot \frac{1}{n} = 1
\end{equation}

\section{Theoretically Necessary of GCD with MTMC}
\label{sec:theoremMTMC}

GCD is a semi-supervised learning scheme and MTMC is theoretically necessary for GCD. We conduct a comprehensive analysis and derivation from High-Dimensional Probability perspectives (with special consideration given to the more general cases where the number of points $P$ is large and the dimension $D$ is high).

GCD aims to cluster the embeddings of samples from the same category as closely as possible, regardless of whether they are from known or unknown classes. That is each cluster center lies on the hypersphere, and the distribution of the centers on the hypersphere is made as uniform as possible. More formally, this goal can be replaced by two definitions in previous studies~\cite{galvez2023role, wang2020understanding}.

$\textbf{Definition 1}$ ($\textit{Perfect Reconstruction}$). A network $f_{\theta}$ is Perfect Reconstruction if $\forall \mathbf{x} \in \mathcal{X}, \forall t^{(1)}, t^{(2)} \in \mathcal{T}, \mathbf{z}^{(1)}=f_\theta\left(t^{(1)}(\mathbf{x})\right)=f_\theta\left(t^{(2)}(\mathbf{x})\right)=\mathbf{z}^{(2)}$, where $\mathcal{T}$ is a set of data augmentations such as color jittering, cropping, flipping, etc. The dataset is $\mathbf{x}_{1:P}$ with $P$ samples, and the set after applying $1:K$ augmentation methods is $t^{(1)}\left(\mathbf{x}_p\right), \ldots, t^{(K)}\left(\mathbf{x}_p\right)$.

$\textbf{Definition 2}$ ($\textit{Perfect Uniformity}$). $p(Z)$ is the distribution over the network representations induced by the data and transformation sampling distributions. If $p(Z)$ is a uniform distribution on the hypersphere, then the network $f_{\theta}$ achieves Perfect Uniformity.

Intuitively, perfect reconstruction means that the network maps all views of the same data to the same embedding, while perfect uniformity means that these embeddings are uniformly distributed on the hypersphere. For brevity, we denote the centroid embedding of the class token representing the $p$-th sample under different augmentations as $\mathbf{z}_p$. We prove the following: A network that simultaneously achieves perfect reconstruction and perfect uniformity achieves a lower bound of what MTMC has, that is, it provides the lowest probability of $\mathcal{L}_{MTMC}$.

$\textbf{Proposition 1}$. $\textit{Suppose that,}$ $\forall p \in[P], \mathbf{c}_p^T \mathbf{c}_p \leq 1$. $\textit{Then,}$ $0 \leq\|C\|_* \leq \sqrt{P \min (P, D)}$.

$\textit{Proof}$. Let $\sigma_1, \ldots, \sigma_{\min (P, D)}$ denote the singular values of $C$, so that $\|C\|_*=\sum_{i=1}^{\min (P, D)} \sigma_i$. The lower bound follows by the fact that singular values are nonnegative. For the upper bound, we have

\begin{equation}
\sum_{i=1}^{\min (P, D)} \sigma_i^2=\operatorname{Tr}\left[C C^T\right]=\sum_{n=1}^P \mathbf{c}_p^T \mathbf{c}_p \leq P
\end{equation}

Then, by Cauchy-Schwarz on the sequences $(1, \ldots, 1)$ and $\left(\sigma_1, \ldots, \sigma_{\min (P, D)}\right)$, we get

\begin{equation}
\sum_{i=1}^{\min (P, D)} \sigma_i \leq \sqrt{\left(\sum_{i=1}^{\min (P, D)} 1\right)\left(\sum_{i=1}^{\min (P, D)} \sigma_i^2\right)} \leq \sqrt{\min (P, D) P} .
\end{equation}

$\textbf{Proposition 2}$. Let $f_\theta$ achieve perfect reconstruction. Then, $\left\|\mathbf{c}_p\right\|_2=1 \forall n$.

$\textit{Proof}$. Because $f_\theta$ achieves perfect reconstruction, $\forall n, \forall t^{(1)}, t^{(2)}, \mathbf{z}_p^{(1)}=\mathbf{z}_p^{(2)}$. Thus $\mathbf{c}_p=$ $(1 / K) \sum_k \mathbf{z}_p^{(k)}=(1 / K) \sum_k \mathbf{z}_p^{(1)}=\mathbf{z}_p^{(1)}$, and since $\left\|\mathbf{z}_p^{(1)}\right\|_2=1$, we have $\left\|\mathbf{c}_p\right\|_2=1$.

$\textbf{Theorem 1}$. Let $f_\theta: \mathcal{X} \rightarrow \mathbb{S}^D$ be a network that achieves perfect reconstruction and perfect uniformity. Then $f_\theta$ achieves the lower bound of $\mathcal{L}_{MTMC}$ with high probability. Specifically:

\begin{equation}
\|C\|_*= \begin{cases}P(1-O(P / D)) & \text { if } P \leq D \\\ \sqrt{P D}(1-O(D / P)) & \text { if } P \geq D\end{cases}
\end{equation}

with high probability in $\min (P, D)$.

This demonstrates that the MTMC loss can be minimized by minimizing the distances of all embeddings corresponding to the same datum and maximizing the distances of all samples' centers.

The above derivations and analyses based on High-Dimensional Probability demonstrate, the $\textbf{theoretical strong correlation}$ of MTMC and GCD (as a type of semi-supervised learning).

\section{More Analysis}
\label{sec:app_analysis}

\subsection{Impact of embedding quality}

In Table~\ref{tab:gcd}, the accuracy gains on the CIFAR100 and Herbarium19 datasets are insignificant. We use this as a starting point to analyze the conflict between enhancing feature completeness and low embedding quality in GCD. DINO, through self-supervision, already has a good feature representation capability, but due to the distribution of data, its embedding quality still be low. One source of low quality is the data size, and the other is data semantics. 

(1) Specifically, when the small-sized CIFAR10 images are interpolated and input into ViT, the high-frequency information is lost. For example, when identifying animal categories, the low-frequency features such as the outline of the animal may be captured relatively well, but the detailed features such as the texture and eyes of the animal (high-frequency features) are difficult to accurately extract. In this case, the model can only cluster through some shortcut information, rather than accurately clustering based on the complete intra-class features.  Since the manifold dimension of the low-frequency features is relatively low, it is unable to fully capture the diversity and complexity within the class. Therefore, enhancing the completeness of the intra-class representation on small-sized data is challenging.

(2) Herbarium19 is a large-scale herbal plant recognition dataset, which is not in the model's training data and inherently cannot provide highly discriminative representations. Additionally, the large number of categories makes the decision boundary more chaotic, and existing GCD schemes cannot cluster well. Therefore, enhancing the completeness of intra-class representation on overly low-quality embeddings is not feasible, as the overlap of feature spaces across categories is too large, and samples within a cluster come from multiple categories.

\begin{figure}[tbp]
  \centering
  \vspace{-0.4cm}
  \subfloat[SimGCD]
  {\includegraphics[width=0.23\textwidth]{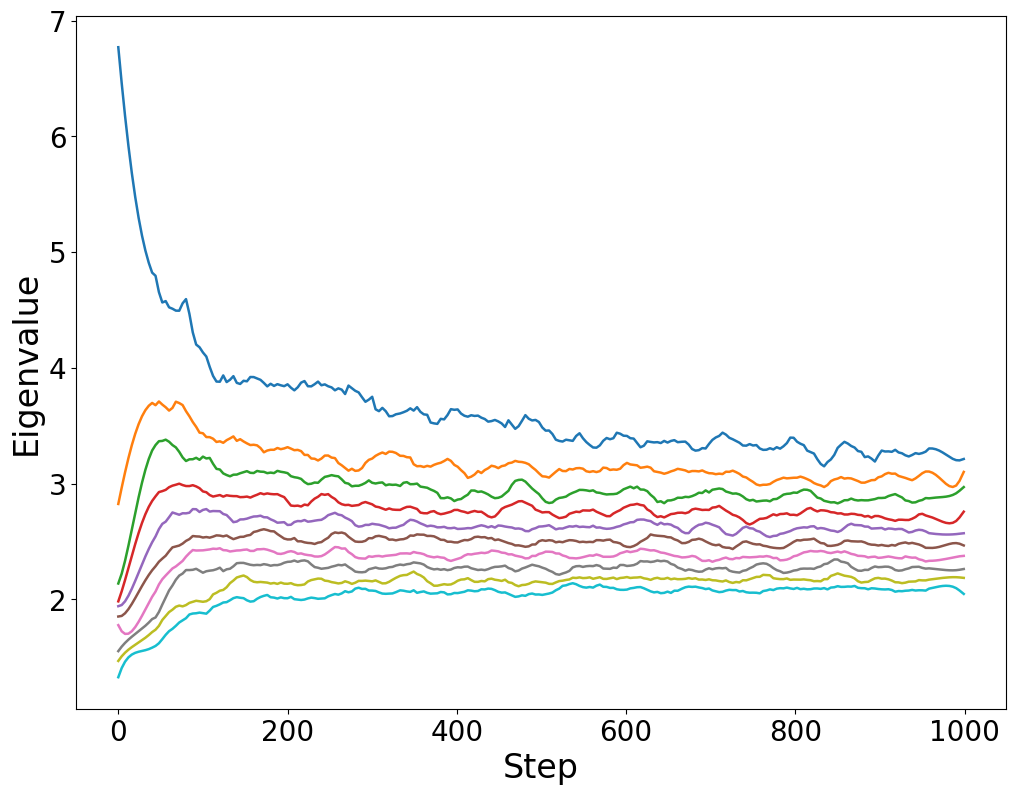}}
  \quad     % 重点就在这，优先横向排列，自动换行
  \subfloat[with CorInfoMax]
  {\includegraphics[width=0.23\textwidth]{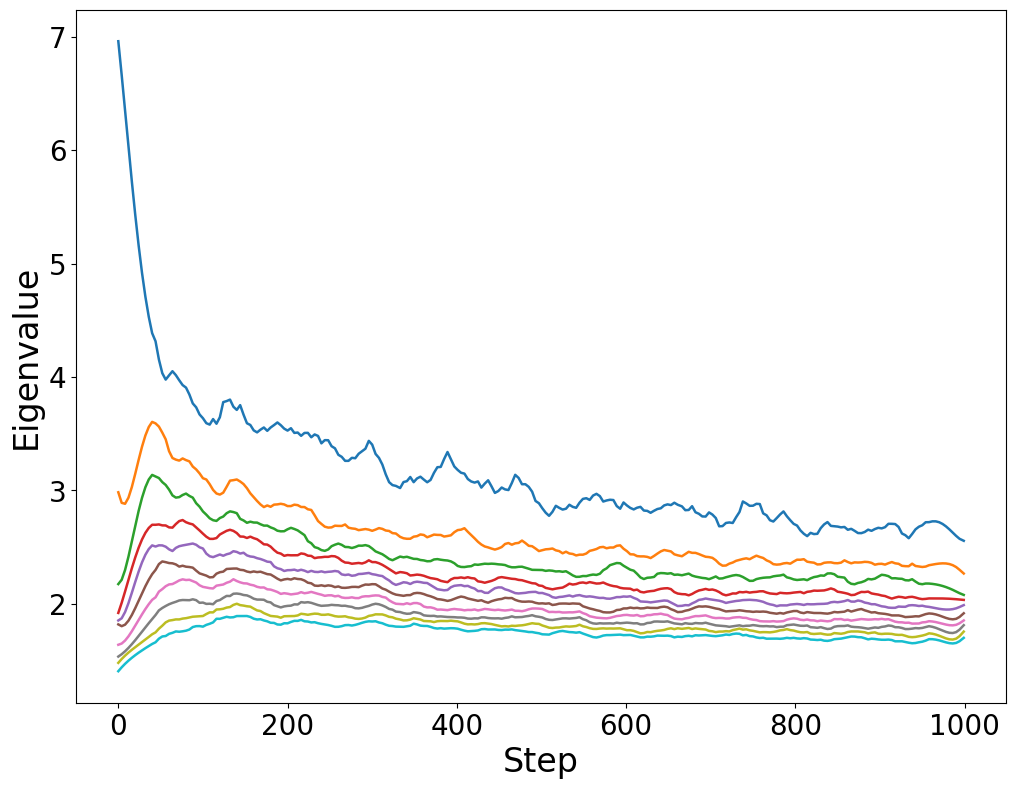}}
  \quad
  \subfloat[with VICReg]
  {\includegraphics[width=0.23\textwidth]{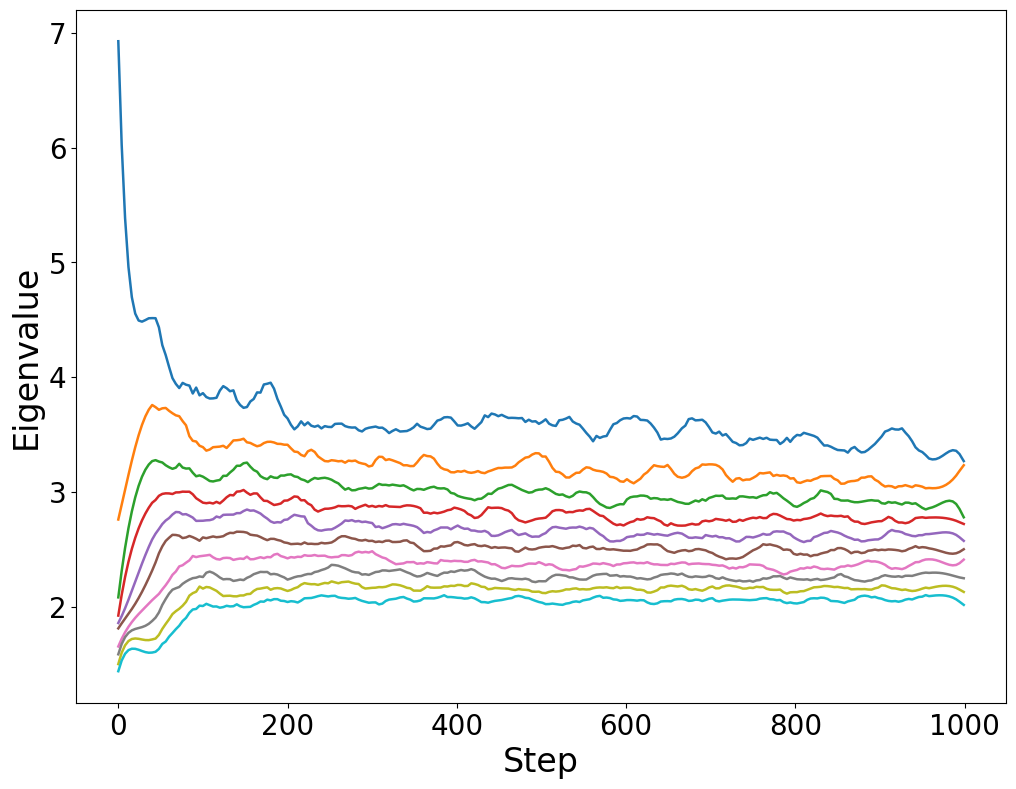}}
  \quad
  \subfloat[with Ours]
  {\includegraphics[width=0.23\textwidth]{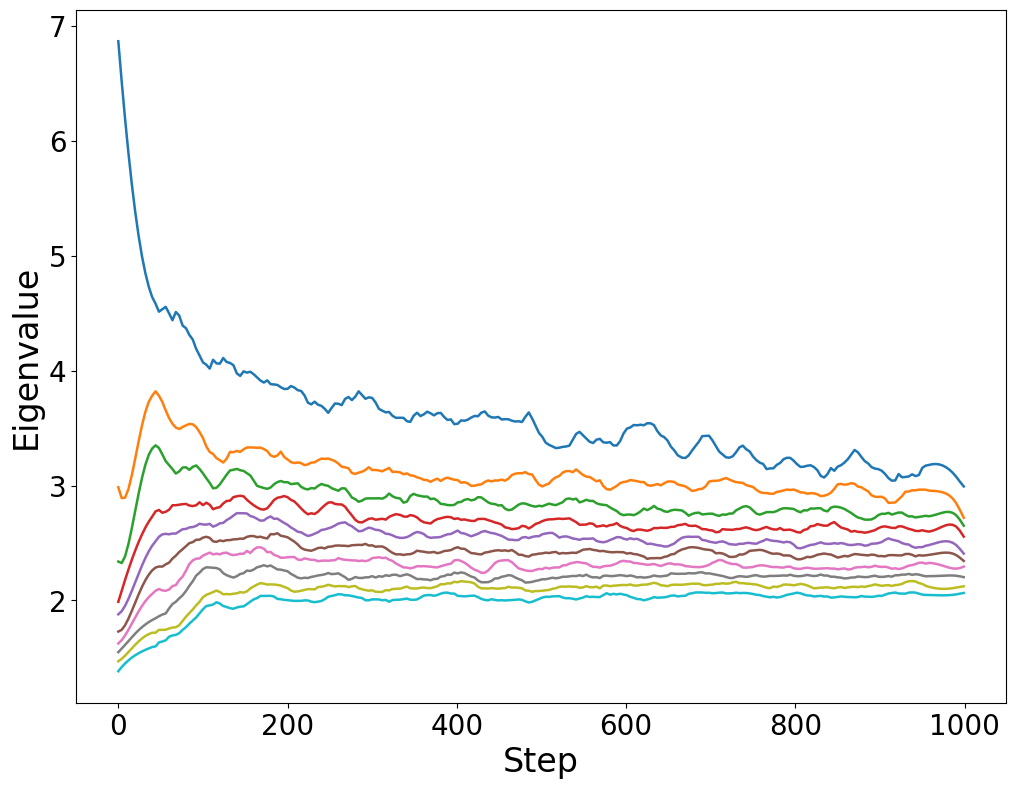}}
  \quad
  \caption{Trends in the top 10 singular values as the number of training steps grows.}
  \label{fig:ablation_svd_lind}
\end{figure}

\subsection{Analysis of VICReg and CorInfoMax}

Compared to the two optimization directions, VICReg and CorInfoMax, MTMC offers a smoother and more uniform convergence of feature values, addressing some key limitations in both methods (Figure \ref{fig:ablation_svd_lind}). VICReg, as a variance-based regularization approach, promotes feature variance and decorrelation but lacks explicit emphasis on intra-class representation completeness. This results in less expressive class boundaries and less effective fine-grained category separation. CorInfoMax, on the other hand, focuses on maximizing mutual information between features and their target distribution but does not sufficiently prevent dimensional collapse or guarantee richer intra-class representations. Both methods, while effective in some contexts, fail to fully capture the complex, high-dimensional structure of the data.

In contrast, MTMC directly targets the manifold capacity of class tokens, ensuring that intra-class representations remain complete and informative. By maximizing the nuclear norm of the class token's singular values, MTMC ensures that feature values converge uniformly, without the collapse seen in other methods. This leads to more robust and accurate clustering, particularly when discovering novel categories. The smooth convergence of MTMC reflects its ability to optimize representation quality while maintaining high inter-class separability, which is critical for open-world learning tasks.

\section{Related Works}
\label{sec:rela}

\subsection{Generalized Category Discovery}
\label{subsec:rela_gcd}

Generalized category discovery~\cite{vaze2022generalized,zhao2023learning,wen2023parametric,choi2024contrastive} is crucial for identifying and classifying both known and new categories in a dataset, expanding beyond traditional supervised learning to recognize new classes not seen during training.
The pioneering work~\cite{vaze2022generalized} establishes a framework that employs semi-supervised k-means clustering. Following this initial proposition, SimGCD~\cite{wen2023parametric} is introduced as a parametric classification approach that utilizes entropy regularization and self-distillation.
Expanding on these concepts, CMS~\cite{choi2024contrastive} is proposed, enhancing representation learning through mean-shift based clustering. Moreover, a deep clustering approach~\cite{zhao2023learning} emerges that dynamically adjusts the number of prototypes during inference, facilitating an adaptive discovery of new categories. Most recently, ActiveGCD~\cite{ma2024active} actively selects samples from unlabeled data to query for labels, with the aim of enhancing the discovery of new categories through an adaptive sampling strategy. Happy~\cite{ma2024happy} explores Continual Generalized Category Discovery (C-GCD), addressing the conflict between discovering new classes and preventing forgetting of old ones through hardness-aware prototype sampling and soft entropy regularization. Each of these contributions addresses the multifaceted challenges of representation learning, category number estimation, and label assignment, redefining the frontiers of open-world learning. Regardless of the flourishing development of GCD, their focus remains on compact clustering, neglecting the integrity of intra-class representation. Our goal is to empower any GCD scheme with concise means to promote the non-collapse representation of each sample, thus shaping more accurate decision boundaries.

\subsection{Dimensional Collapse}
\label{subsec:rela_collapse}

 This Dimensional collapse~\cite{grill2020bootstrap,caron2020unsupervised,shi2023understanding,jing2021understanding} occurs when the learned embeddings tend to concentrate within a lower-dimensional subspace rather than dispersing throughout the entire embedding space, thereby limiting the representations' capacity for diversity and expressiveness. DirectCLR~\cite{jing2021understanding} presents a direct optimization of the representation space, sidestepping the need for a trainable projector, which inherently mitigates the risk of dimensional collapse by promoting a more even distribution of embeddings across the space. Complementing this, the whitening approach~\cite{tao2024breaking} standardizes covariance matrices through whitening techniques, ensuring that each dimension contributes equally to the representation, thus preventing any subset of dimensions from dominating the learning process. Similarly, the non-contrastive learning objective~\cite{chen2024towards} for collaborative filtering avoids data augmentation and negative sampling, focusing on alignment and compactness within the embedding space to prevent dimensional collapse. The Bregman matrix divergence~\cite{zhang2024geometric} further fortifies the fight against dimensional collapse by minimizing the distance between covariance matrices and the identity matrix, ensuring a uniform distribution of embeddings and directly countering the concentration of information along certain dimensions. Moreover, random orthogonal projection image modeling~\cite{haghighat2023pre} provides a preventative measure against dimensional collapse by modeling images with random orthogonal projections, which promotes the exploration of a wide range of features and discourages the concentration on a limited subset of dimensions. Rather than directly addressing the issue of dimensional collapse, we focus on maximizing token manifold capacity to align the radius and dimensions of the manifold with the rich distribution of the real world. This approach also unravels the sample-level dimensional collapse.

%% file: neurips_2025.bbl
\begin{thebibliography}{10}

\bibitem{bardes2022vicreg}
Adrien Bardes, Jean Ponce, and Yann LeCun.
\newblock {VICR}eg: Variance-invariance-covariance regularization for self-supervised learning.
\newblock In {\em International Conference on Learning Representations}, 2022.

\bibitem{boes2019neumann}
Paul Boes, Jens Eisert, Rodrigo Gallego, Markus~P M{\"u}ller, and Henrik Wilming.
\newblock Von neumann entropy from unitarity.
\newblock {\em Physical review letters}, 122(21):210402, 2019.

\bibitem{cao2022openworld}
Kaidi Cao, Maria Brbic, and Jure Leskovec.
\newblock Open-world semi-supervised learning.
\newblock In {\em International Conference on Learning Representations}, 2022.

\bibitem{caron2020unsupervised}
Mathilde Caron, Ishan Misra, Julien Mairal, Priya Goyal, Piotr Bojanowski, and Armand Joulin.
\newblock Unsupervised learning of visual features by contrasting cluster assignments.
\newblock {\em Advances in neural information processing systems}, 33:9912--9924, 2020.

\bibitem{caron2021emerging}
Mathilde Caron, Hugo Touvron, Ishan Misra, Herv{\'e} J{\'e}gou, Julien Mairal, Piotr Bojanowski, and Armand Joulin.
\newblock Emerging properties in self-supervised vision transformers.
\newblock In {\em Proceedings of the IEEE/CVF international conference on computer vision}, pages 9650--9660, 2021.

\bibitem{chen2024towards}
Huiyuan Chen, Vivian Lai, Hongye Jin, Zhimeng Jiang, Mahashweta Das, and Xia Hu.
\newblock Towards mitigating dimensional collapse of representations in collaborative filtering.
\newblock In {\em Proceedings of the 17th ACM International Conference on Web Search and Data Mining}, pages 106--115, 2024.

\bibitem{chen2020simple}
Ting Chen, Simon Kornblith, Mohammad Norouzi, and Geoffrey Hinton.
\newblock A simple framework for contrastive learning of visual representations.
\newblock In {\em International conference on machine learning}, pages 1597--1607. PMLR, 2020.

\bibitem{chiaroni2023parametric}
Florent Chiaroni, Jose Dolz, Ziko~Imtiaz Masud, Amar Mitiche, and Ismail Ben~Ayed.
\newblock Parametric information maximization for generalized category discovery.
\newblock In {\em Proceedings of the IEEE/CVF International Conference on Computer Vision}, pages 1729--1739, 2023.

\bibitem{choi2019utilizing}
Daeyoung Choi and Wonjong Rhee.
\newblock Utilizing class information for deep network representation shaping.
\newblock In {\em Proceedings of the AAAI Conference on Artificial Intelligence}, volume~33, pages 3396--3403, 2019.

\bibitem{choi2024contrastive}
Sua Choi, Dahyun Kang, and Minsu Cho.
\newblock Contrastive mean-shift learning for generalized category discovery.
\newblock In {\em Proceedings of the IEEE/CVF Conference on Computer Vision and Pattern Recognition}, pages 23094--23104, 2024.

\bibitem{cogswell2015reducing}
Michael Cogswell, Faruk Ahmed, Ross Girshick, Larry Zitnick, and Dhruv Batra.
\newblock Reducing overfitting in deep networks by decorrelating representations.
\newblock {\em arXiv preprint arXiv:1511.06068}, 2015.

\bibitem{dosovitskiy2020image}
Alexey Dosovitskiy.
\newblock An image is worth 16x16 words: Transformers for image recognition at scale.
\newblock {\em arXiv preprint arXiv:2010.11929}, 2020.

\bibitem{duchi2007derivations}
John Duchi.
\newblock Derivations for linear algebra and optimization.
\newblock {\em Berkeley, California}, 3(1):2325--5870, 2007.

\bibitem{fini2021unified}
Enrico Fini, Enver Sangineto, St{\'e}phane Lathuili{\`e}re, Zhun Zhong, Moin Nabi, and Elisa Ricci.
\newblock A unified objective for novel class discovery.
\newblock In {\em Proceedings of the IEEE/CVF International Conference on Computer Vision}, pages 9284--9292, 2021.

\bibitem{galvez2023role}
Borja~Rodr{\i}guez G{\'a}lvez, Arno Blaas, Pau Rodr{\'\i}guez, Adam Golinski, Xavier Suau, Jason Ramapuram, Dan Busbridge, and Luca Zappella.
\newblock The role of entropy and reconstruction in multi-view self-supervised learning.
\newblock In {\em International Conference on Machine Learning}, pages 29143--29160. PMLR, 2023.

\bibitem{geirhos2018imagenettrained}
Robert Geirhos, Patricia Rubisch, Claudio Michaelis, Matthias Bethge, Felix~A. Wichmann, and Wieland Brendel.
\newblock Imagenet-trained {CNN}s are biased towards texture; increasing shape bias improves accuracy and robustness.
\newblock In {\em International Conference on Learning Representations}, 2019.

\bibitem{geng2020recent}
Chuanxing Geng, Sheng-jun Huang, and Songcan Chen.
\newblock Recent advances in open set recognition: A survey.
\newblock {\em IEEE transactions on pattern analysis and machine intelligence}, 43(10):3614--3631, 2020.

\bibitem{golub1971singular}
Gene~H Golub and Christian Reinsch.
\newblock Singular value decomposition and least squares solutions.
\newblock In {\em Handbook for Automatic Computation: Volume II: Linear Algebra}, pages 134--151. Springer, 1971.

\bibitem{grill2020bootstrap}
Jean-Bastien Grill, Florian Strub, Florent Altch{\'e}, Corentin Tallec, Pierre Richemond, Elena Buchatskaya, Carl Doersch, Bernardo Avila~Pires, Zhaohan Guo, Mohammad Gheshlaghi~Azar, et~al.
\newblock Bootstrap your own latent-a new approach to self-supervised learning.
\newblock {\em Advances in neural information processing systems}, 33:21271--21284, 2020.

\bibitem{haghighat2023pre}
Maryam Haghighat, Peyman Moghadam, Shaheer Mohamed, and Piotr Koniusz.
\newblock Pre-training with random orthogonal projection image modeling.
\newblock {\em arXiv preprint arXiv:2310.18737}, 2023.

\bibitem{han2020automatically}
Kai Han, Sylvestre-Alvise Rebuffi, Sebastien Ehrhardt, Andrea Vedaldi, and Andrew Zisserman.
\newblock Automatically discovering and learning new visual categories with ranking statistics.
\newblock {\em arXiv preprint arXiv:2002.05714}, 2020.

\bibitem{han2019learning}
Kai Han, Andrea Vedaldi, and Andrew Zisserman.
\newblock Learning to discover novel visual categories via deep transfer clustering.
\newblock In {\em Proceedings of the IEEE/CVF International Conference on Computer Vision}, pages 8401--8409, 2019.

\bibitem{hua2021feature}
Tianyu Hua, Wenxiao Wang, Zihui Xue, Sucheng Ren, Yue Wang, and Hang Zhao.
\newblock On feature decorrelation in self-supervised learning.
\newblock In {\em Proceedings of the IEEE/CVF International Conference on Computer Vision}, pages 9598--9608, 2021.

\bibitem{isik2023information}
Berivan Isik, Victor Lecomte, Rylan Schaeffer, Yann LeCun, Mikail Khona, Ravid Shwartz-Ziv, Sanmi Koyejo, and Andrey Gromov.
\newblock An information-theoretic understanding of maximum manifold capacity representations.
\newblock In {\em UniReps: the First Workshop on Unifying Representations in Neural Models}, 2023.

\bibitem{jing2021understanding}
Li~Jing, Pascal Vincent, Yann LeCun, and Yuandong Tian.
\newblock Understanding dimensional collapse in contrastive self-supervised learning.
\newblock {\em arXiv preprint arXiv:2110.09348}, 2021.

\bibitem{khosla2020supervised}
Prannay Khosla, Piotr Teterwak, Chen Wang, Aaron Sarna, Yonglong Tian, Phillip Isola, Aaron Maschinot, Ce~Liu, and Dilip Krishnan.
\newblock Supervised contrastive learning.
\newblock {\em Advances in neural information processing systems}, 33:18661--18673, 2020.

\bibitem{krause20133d}
Jonathan Krause, Michael Stark, Jia Deng, and Li~Fei-Fei.
\newblock 3d object representations for fine-grained categorization.
\newblock In {\em Proceedings of the IEEE international conference on computer vision workshops}, pages 554--561, 2013.

\bibitem{krizhevsky2009learning}
Alex Krizhevsky and Geoffrey Hinton.
\newblock Cifar-10 dataset.
\newblock \url{https://www.cs.toronto.edu/~kriz/cifar.html}, 2009.
\newblock Accessed: 2025-05-20.

\bibitem{kuhn1955hungarian}
Harold~W Kuhn.
\newblock The hungarian method for the assignment problem.
\newblock {\em Naval research logistics quarterly}, 2(1-2):83--97, 1955.

\bibitem{liu2017sphereface}
Weiyang Liu, Yandong Wen, Zhiding Yu, Ming Li, Bhiksha Raj, and Le~Song.
\newblock Sphereface: Deep hypersphere embedding for face recognition.
\newblock In {\em Proceedings of the IEEE conference on computer vision and pattern recognition}, pages 212--220, 2017.

\bibitem{ma1994frobenius}
Changxue Ma, Yves Kamp, and Lei~F Willems.
\newblock A frobenius norm approach to glottal closure detection from the speech signal.
\newblock {\em IEEE Transactions on Speech and Audio Processing}, 2(2):258--265, 1994.

\bibitem{ma2025protogcd}
Shijie Ma, Fei Zhu, Xu-Yao Zhang, and Cheng-Lin Liu.
\newblock Protogcd: Unified and unbiased prototype learning for generalized category discovery.
\newblock {\em IEEE Transactions on Pattern Analysis and Machine Intelligence}, 2025.

\bibitem{ma2024happy}
Shijie Ma, Fei Zhu, Zhun Zhong, Wenzhuo Liu, Xu-Yao Zhang, and Cheng-Lin Liu.
\newblock Happy: A debiased learning framework for continual generalized category discovery.
\newblock {\em arXiv preprint arXiv:2410.06535}, 2024.

\bibitem{ma2024active}
Shijie Ma, Fei Zhu, Zhun Zhong, Xu-Yao Zhang, and Cheng-Lin Liu.
\newblock Active generalized category discovery.
\newblock In {\em Proceedings of the IEEE/CVF Conference on Computer Vision and Pattern Recognition}, pages 16890--16900, 2024.

\bibitem{maji2013fine}
Subhransu Maji, Esa Rahtu, Juho Kannala, Matthew Blaschko, and Andrea Vedaldi.
\newblock Fine-grained visual classification of aircraft.
\newblock {\em arXiv preprint arXiv:1306.5151}, 2013.

\bibitem{marshall1979inequalities}
AW~Marshall.
\newblock Inequalities: Theory of majorization and its applications, 1979.

\bibitem{mettes2019hyperspherical}
Pascal Mettes, Elise Van~der Pol, and Cees Snoek.
\newblock Hyperspherical prototype networks.
\newblock {\em Advances in neural information processing systems}, 32, 2019.

\bibitem{corinfomax}
Serdar Ozsoy, Shadi Hamdan, Sercan Arik, Deniz Yuret, and Alper Erdogan.
\newblock Self-supervised learning with an information maximization criterion.
\newblock In S.~Koyejo, S.~Mohamed, A.~Agarwal, D.~Belgrave, K.~Cho, and A.~Oh, editors, {\em Advances in Neural Information Processing Systems}, volume~35, pages 35240--35253. Curran Associates, Inc., 2022.

\bibitem{parkhi2015deep}
Omkar Parkhi, Andrea Vedaldi, and Andrew Zisserman.
\newblock Deep face recognition.
\newblock In {\em BMVC 2015-Proceedings of the British Machine Vision Conference 2015}. British Machine Vision Association, 2015.

\bibitem{peng2016connections}
Xi~Peng, Canyi Lu, Zhang Yi, and Huajin Tang.
\newblock Connections between nuclear-norm and frobenius-norm-based representations.
\newblock {\em IEEE transactions on neural networks and learning systems}, 29(1):218--224, 2016.

\bibitem{petz2001entropy}
D{\'e}nes Petz.
\newblock Entropy, von neumann and the von neumann entropy: Dedicated to the memory of alfred wehrl.
\newblock In {\em John von Neumann and the foundations of quantum physics}, pages 83--96. Springer, 2001.

\bibitem{pu2023dynamic}
Nan Pu, Zhun Zhong, and Nicu Sebe.
\newblock Dynamic conceptional contrastive learning for generalized category discovery.
\newblock In {\em Proceedings of the IEEE/CVF conference on computer vision and pattern recognition}, pages 7579--7588, 2023.

\bibitem{sarker2021machine}
Iqbal~H Sarker.
\newblock Machine learning: Algorithms, real-world applications and research directions.
\newblock {\em SN computer science}, 2(3):160, 2021.

\bibitem{schaeffer2024towards}
Rylan Schaeffer, Victor Lecomte, Dhruv~Bhandarkar Pai, Andres Carranza, Berivan Isik, Alyssa Unell, Mikail Khona, Thomas Yerxa, Yann LeCun, SueYeon Chung, et~al.
\newblock Towards an improved understanding and utilization of maximum manifold capacity representations.
\newblock {\em arXiv preprint arXiv:2406.09366}, 2024.

\bibitem{shi2023understanding}
Yujun Shi, Jian Liang, Wenqing Zhang, Chuhui Xue, Vincent~YF Tan, and Song Bai.
\newblock Understanding and mitigating dimensional collapse in federated learning.
\newblock {\em IEEE Transactions on Pattern Analysis and Machine Intelligence}, 2023.

\bibitem{souvenir2005manifold}
Richard Souvenir and Robert Pless.
\newblock Manifold clustering.
\newblock In {\em Tenth IEEE International Conference on Computer Vision (ICCV'05) Volume 1}, volume~1, pages 648--653. IEEE, 2005.

\bibitem{tan2019herbarium}
Kiat~Chuan Tan, Yulong Liu, Barbara Ambrose, Melissa Tulig, and Serge Belongie.
\newblock The herbarium challenge 2019 dataset.
\newblock {\em arXiv preprint arXiv:1906.05372}, 2019.

\bibitem{tao2024breaking}
Yang Tao, Kai Guo, Yizhen Zheng, Shirui Pan, Xiaofeng Cao, and Yi~Chang.
\newblock Breaking the curse of dimensional collapse in graph contrastive learning: A whitening perspective.
\newblock {\em Information Sciences}, 657:119952, 2024.

\bibitem{thomas2006elements}
MTCAJ Thomas and A~Thomas Joy.
\newblock {\em Elements of information theory}.
\newblock Wiley-Interscience, 2006.

\bibitem{vaze2022generalized}
Sagar Vaze, Kai Han, Andrea Vedaldi, and Andrew Zisserman.
\newblock Generalized category discovery.
\newblock In {\em Proceedings of the IEEE/CVF Conference on Computer Vision and Pattern Recognition}, pages 7492--7501, 2022.

\bibitem{vaze2021open}
Sagar Vaze, Kai Han, Andrea Vedaldi, and Andrew Zisserman.
\newblock Open-set recognition: A good closed-set classifier is all you need.
\newblock In {\em International Conference on Learning Representations}, 2022.

\bibitem{vershynin2018high}
Roman Vershynin.
\newblock {\em High-dimensional probability: An introduction with applications in data science}, volume~47.
\newblock Cambridge university press, 2018.

\bibitem{wah2011caltech}
C.~Wah, N.~Rasiwasia, D.~Hsu, J.~Yao, L.~Li, and G.~Mori.
\newblock Caltech-ucsd birds 200-2011 (cub-200-2011).
\newblock \url{http://www.vision.caltech.edu/visipedia/CUB-200.html}, 2011.
\newblock Accessed: 2025-05-20.

\bibitem{wang2020understanding}
Tongzhou Wang and Phillip Isola.
\newblock Understanding contrastive representation learning through alignment and uniformity on the hypersphere.
\newblock In {\em International conference on machine learning}, pages 9929--9939. PMLR, 2020.

\bibitem{ward1963hierarchical}
Joe~H Ward~Jr.
\newblock Hierarchical grouping to optimize an objective function.
\newblock {\em Journal of the American statistical association}, 58(301):236--244, 1963.

\bibitem{weiss2016survey}
Karl Weiss, Taghi~M Khoshgoftaar, and DingDing Wang.
\newblock A survey of transfer learning.
\newblock {\em Journal of Big data}, 3:1--40, 2016.

\bibitem{wen2023parametric}
Xin Wen, Bingchen Zhao, and Xiaojuan Qi.
\newblock Parametric classification for generalized category discovery: A baseline study.
\newblock In {\em Proceedings of the IEEE/CVF International Conference on Computer Vision}, pages 16590--16600, 2023.

\bibitem{wu2024towards}
Jianzong Wu, Xiangtai Li, Shilin Xu, Haobo Yuan, Henghui Ding, Yibo Yang, Xia Li, Jiangning Zhang, Yunhai Tong, Xudong Jiang, et~al.
\newblock Towards open vocabulary learning: A survey.
\newblock {\em IEEE Transactions on Pattern Analysis and Machine Intelligence}, 2024.

\bibitem{xiong2016regularizing}
Wei Xiong, Bo~Du, Lefei Zhang, Ruimin Hu, and Dacheng Tao.
\newblock Regularizing deep convolutional neural networks with a structured decorrelation constraint.
\newblock In {\em 2016 IEEE 16th international conference on data mining (ICDM)}, pages 519--528. IEEE, 2016.

\bibitem{yerxa2023learning}
Thomas Yerxa, Yilun Kuang, Eero Simoncelli, and SueYeon Chung.
\newblock Learning efficient coding of natural images with maximum manifold capacity representations.
\newblock {\em Advances in Neural Information Processing Systems}, 36:24103--24128, 2023.

\bibitem{zbontar2021barlow}
Jure Zbontar, Li~Jing, Ishan Misra, Yann LeCun, and St{\'e}phane Deny.
\newblock Barlow twins: Self-supervised learning via redundancy reduction.
\newblock In {\em International conference on machine learning}, pages 12310--12320. PMLR, 2021.

\bibitem{zhang2023promptcal}
Sheng Zhang, Salman Khan, Zhiqiang Shen, Muzammal Naseer, Guangyi Chen, and Fahad~Shahbaz Khan.
\newblock Promptcal: Contrastive affinity learning via auxiliary prompts for generalized novel category discovery.
\newblock In {\em Proceedings of the IEEE/CVF Conference on Computer Vision and Pattern Recognition}, pages 3479--3488, 2023.

\bibitem{zhang2024geometric}
Yifei Zhang, Hao Zhu, Zixing Song, Yankai Chen, Xinyu Fu, Ziqiao Meng, Piotr Koniusz, and Irwin King.
\newblock Geometric view of soft decorrelation in self-supervised learning.
\newblock In {\em Proceedings of the 30th ACM SIGKDD Conference on Knowledge Discovery and Data Mining}, pages 4338--4349, 2024.

\bibitem{zhao2023learning}
Bingchen Zhao, Xin Wen, and Kai Han.
\newblock Learning semi-supervised gaussian mixture models for generalized category discovery.
\newblock In {\em Proceedings of the IEEE/CVF International Conference on Computer Vision}, pages 16623--16633, 2023.

\bibitem{zhou2022domain}
Kaiyang Zhou, Ziwei Liu, Yu~Qiao, Tao Xiang, and Chen~Change Loy.
\newblock Domain generalization: A survey.
\newblock {\em IEEE Transactions on Pattern Analysis and Machine Intelligence}, 45(4):4396--4415, 2022.

\end{thebibliography}
